\documentclass{article}

\usepackage[final]{corl_2022} %
\usepackage{siunitx} %

\usepackage{booktabs} %
\usepackage{tikz} %
\usepackage{algorithm}
\usepackage{algpseudocode}
\usepackage{comment}

\usepackage{amsmath,amsfonts,bm, bbm}

\def\eqref#1{equation~\ref{#1}}
\def\Eqref#1{Equation~\ref{#1}}

\def\1{\bm{1}}

\def\rd{{\textnormal{d}}}

\def\rl{{\textnormal{l}}}

\def\vzero{{\bm{0}}}
\def\vone{{\bm{1}}}
\def\vmu{{\bm{\mu}}}
\def\vnu{{\bm{\nu}}}
\def\vtheta{{\bm{\theta}}}
\def\va{{\bm{a}}}

\def\vc{{\bm{c}}}

\def\vf{{\bm{f}}}

\def\vk{{\bm{k}}}

\def\vn{{\bm{n}}}

\def\vr{{\bm{r}}}
\def\vs{{\bm{s}}}
\def\vt{{\bm{t}}}
\def\vu{{\bm{u}}}
\def\vv{{\bm{v}}}
\def\vw{{\bm{w}}}
\def\vx{{\bm{x}}}
\def\vy{{\bm{y}}}

\def\mA{{\bm{A}}}
\def\mB{{\bm{B}}}

\def\mH{{\bm{H}}}
\def\mI{{\bm{I}}}

\def\mK{{\bm{K}}}
\def\mL{{\bm{L}}}
\def\mM{{\bm{M}}}

\def\mR{{\bm{R}}}
\def\mS{{\bm{S}}}

\def\mW{{\bm{W}}}
\def\mX{{\bm{X}}}

\def\mLambda{{\bm{\Lambda}}}
\def\mSigma{{\bm{\Sigma}}}

\DeclareMathAlphabet{\mathsfit}{\encodingdefault}{\sfdefault}{m}{sl}
\SetMathAlphabet{\mathsfit}{bold}{\encodingdefault}{\sfdefault}{bx}{n}

\def\gE{{\mathcal{E}}}

\def\gG{{\mathcal{G}}}

\def\gL{{\mathcal{L}}}
\def\gM{{\mathcal{M}}}
\def\gN{{\mathcal{N}}}
\def\gO{{\mathcal{O}}}
\def\gP{{\mathcal{P}}}
\def\gQ{{\mathcal{Q}}}

\def\gX{{\mathcal{X}}}

\def\sD{{\mathbb{D}}}

\def\sP{{\mathbb{P}}}
\def\sQ{{\mathbb{Q}}}
\def\sR{{\mathbb{R}}}

\newcommand{\E}{\mathbb{E}}

\newcommand{\KL}{\sD_{\mathrm{KL}}}

\newcommand{\tran}{^\top}
\newcommand{\inv}{^{-1}}

\DeclareMathOperator*{\argmax}{arg\,max}
\DeclareMathOperator*{\argmin}{arg\,min}

\usepackage{amsthm}
\newtheorem{theorem}{Theorem}

\newtheorem{lemma}{Lemma}
\newtheorem{definition}{Definition}
\newtheorem{example}{Example}

\usepackage{thmtools} 
\usepackage{thm-restate}

\usepackage{xspace}
\usepackage{pgfplots}
\DeclareRobustCommand{\blackdashed}{\raisebox{2pt}{\protect\tikz{\draw[black,dashed,line width=0.9pt](0,0) -- (5mm,0);}}\xspace}
\DeclareRobustCommand{\reddashed}{\raisebox{2pt}{\protect\tikz{\draw[red,dashed,line width=1.2pt](0,0) -- (5mm,0);}}\xspace}
\DeclareRobustCommand{\cyansolid}{\raisebox{2pt}{\protect\tikz{\draw[cyan,solid,line width=1.2pt](0,0) -- (5mm,0);}}\xspace}
\DeclareRobustCommand{\bluesolid}{\raisebox{2pt}{\protect\tikz{\draw[blue,solid,line width=1.2pt](0,0) -- (5mm,0);}}\xspace}

\usepackage{pifont}%

\newcommand{\ppi}{\textsc{ppi}\xspace}
\newcommand{\mpc}{\textsc{mpc}\xspace}
\newcommand{\reps}{\textsc{reps}\xspace}
\newcommand{\essps}{\textsc{essps}\xspace}
\newcommand{\lbps}{\textsc{lbps}\xspace}
\newcommand{\mppi}{\textsc{mppi}\xspace}
\newcommand{\pitwo}{\textsc{pi}$^2$\xspace}
\newcommand{\stomp}{\textsc{stomp}\xspace}
\newcommand{\gpmp}{\textsc{gpmp}\xspace}
\newcommand{\more}{\textsc{more}\xspace}
\newcommand{\cem}{\textsc{cem}\xspace}
\newcommand{\icem}{i\textsc{cem}\xspace}
\newcommand{\snis}{\textsc{snis}\xspace}

\newcommand{\rbf}{\textsc{rbf}\xspace}
\newcommand{\rff}{\textsc{rff}\xspace}
\newcommand{\qrff}{\textsc{qrff}\xspace}
\newcommand{\ess}{\textsc{ess}\xspace}
\newcommand{\gp}{\textsc{gp}\xspace}

\renewcommand{\rl}{\textsc{rl}\xspace}
\newcommand{\lgds}{\textsc{lgds}\xspace}
\newcommand{\promp}{\textsc{p}ro\textsc{mp}\xspace}
\newcommand{\se}{\textsc{se}\xspace}
\newcommand{\dmd}{\textsc{dmd}\xspace}

\newcommand{\kl}{\textsc{kl}\xspace}
\newcommand{\is}{\textsc{is}\xspace}

\newcommand{\hsu}{\texttt{HumanoidStandup-v2}\xspace}
\newcommand{\door}{\texttt{door-v0}\xspace}
\newcommand{\hammer}{\texttt{hammer-v0}\xspace}
\newcommand{\mujoco}{\texttt{MuJoCo}\xspace}

\definecolor{darkgray176}{RGB}{176,176,176}
\definecolor{darkturquoise0191191}{RGB}{0,191,191}
\definecolor{darkviolet1910191}{RGB}{191,0,191}
\definecolor{goldenrod1911910}{RGB}{191,191,0}
\definecolor{green01270}{RGB}{0,127,0}

\title{
Inferring Smooth Control:
Monte Carlo\\
Posterior Policy Iteration with Gaussian Processes
}

\author{
  Joe Watson$^1$
  \hspace{0.3cm}
  Jan Peters$^{1234}$\\
  $^1$Department of Computer Science,
  Technical University of Darmstadt\\
  $^2$Centre for Cognitive Science,
  Technical University of Darmstadt\\
  $^3$German Research Center for AI
  \hspace{0.3cm}
  $^4$Hessian.AI\\
  \texttt{\{joe,jan\}@robot-learning.de}
  \vspace{-2em}
}
  
\begin{document}

\maketitle

\begin{abstract}
Monte Carlo methods have become increasingly relevant for control of non-differentiable systems, approximate dynamics models and learning from data.
These methods scale to high-dimensional spaces and are effective at the non-convex optimizations often seen in robot learning.
We look at sample-based methods from the perspective of inference-based control, specifically posterior policy iteration.
From this perspective, we highlight how Gaussian noise priors produce rough control actions that are unsuitable for physical robot deployment.
Considering smoother Gaussian process priors, as used in episodic reinforcement learning and motion planning, we demonstrate how smoother model predictive control can be achieved using online sequential inference. 
This inference is realized through an efficient factorization of the action distribution and a novel means of optimizing the likelihood temperature to improve importance sampling accuracy.
We evaluate this approach on several high-dimensional robot control tasks, matching the sample efficiency of prior heuristic methods while also ensuring smoothness.
Simulation results can be seen at \href{https://monte-carlo-ppi.github.io/}{\ttfamily monte-carlo-ppi.github.io}.
\end{abstract}

\keywords{approximate inference, policy search, model predictive control}

\begin{figure}[!h]
	\centering
	\hspace{-2em}
		\begin{tabular}{c@{\hspace{0cm}}c@{\hspace{0cm}}c@{\hspace{0
						cm}}}

				\hspace{-0.925em}
				\includegraphics[width=.31\textwidth]{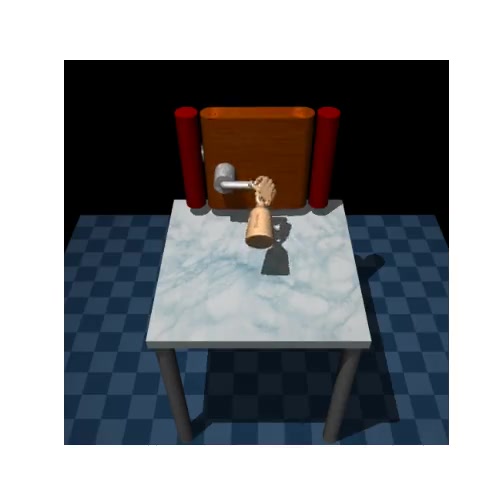} &
				\includegraphics{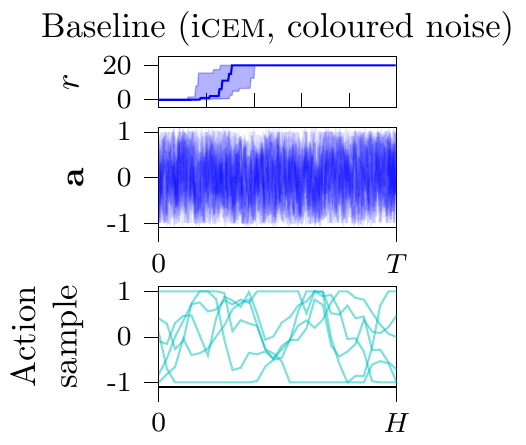} &
				\includegraphics{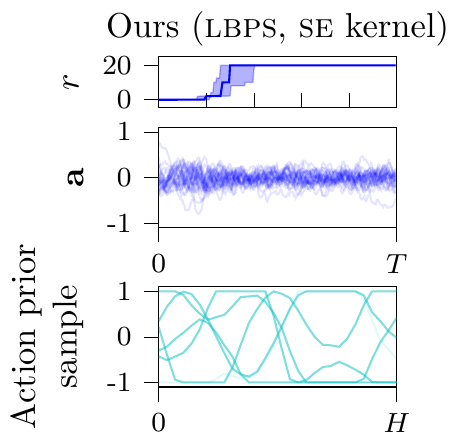}
			\end{tabular}
	\vspace{-1.0em}
	\caption{High-dimensional, contact-rich tasks such as manipulation (left) can be performed effectively using sample-based model predictive control.
	While prior work uses correlated actuator noise to improve sample-efficiency and exploration, these methods do not preserve the smoothness in the downstream actuation $\va$, resulting in aggressive control (center).
	We use smooth Gaussian process priors to infer posterior actions (right), which preserves smoothness while maintaining performance and sample efficiency, as both are using only 32 samples. Rewards $r$ show quartiles over 25 seeds.}
	\label{fig:pull}
	\vspace{-0.5em}
\end{figure}

\section{Introduction}\label{sec:intro}
\vspace{-0.5em}
Learning robot control requires optimization to be performed on sampled transitions of the environment \cite{sutton2018reinforcement}.
Monte Carlo methods \citep{mcbook} provide a principled means to approach such algorithms, bridging black-box optimization and approximate inference techniques.
These methods have been adopted extensively by the community for their impressive simulated \citep{chua2018deep,PinneriiCEM,okada2020variational,lutter2021learning} and real-world \citep{stulp11, stomp, daniel2016hierarchical, williams2018information, nagabandi2020deep,yang2020data, bhardwaj2022storm, lutter2020differentiable, carius2021constrained} robot learning results.
Their appeal includes requiring only function evaluations of the dynamics and objective, so can be applied to complex environments with minimal overhead (Figure \ref{fig:pull}).
Moreover, their stochastic nature also avoids issues with local minima that occur with gradient-based solvers \citep{wierstra14a, abdolmaleki2015model}. 
Finally, while Monte Carlo sampling is expensive, shooting methods can be effectively parallelized across processes and the advent of simulations on \textsc{gpu}s also provides a means of acceleration \citep{williams2017model, bhardwaj2022storm}.
However, some aspects of black-box optimization are open to criticism.
Sample-based solvers such as the cross-entropy method (\cem) \citep{cem} appear `wasteful', ignoring computation by throwing away the majority of samples, while others enforce high-entropy search distributions to avoid premature convergence \citep{williams2017model}. 
Moreover, many design decisions and hyperparameters are heuristic in nature, which is undesirable from both the user- and research perspective when interpreting, tuning or advancing these methods.

In this work, we consider Monte Carlo optimal control through the broader perspective of inference-based control \cite{
Dayan97,
06-toussaint-ICML,
Peters_PICML_2007, toussaint2009robot, kappen2013, MukadamCYB17, levine2018reinforcement, watson2021stochastic}, where optimization is achieved through importance sampling~\cite{kappen2016adaptive}.
This approach covers settings such as policy search \cite{deisenroth2013survey}, motion planning \cite{stomp, Mukadam_mp} and model predictive control (\mpc) \cite{williams2017model}.
From this view point, we highlight two key design decisions: the likelihood temperature and the distribution over action sequences. 
An adaptive temperature scheme is crucial for controlling the optimization behavior across objectives and distributions, but in many methods this aspect is ignored or opaque.
Moreover, correlated action sequences are equally crucial for performing effective exploration and control in practical settings.
\emph{Smoothness}, arising from such correlations, is an aspect of human motion \citep{flash1985coordination}. 
Smooth priors have taken many forms across domains, such as movement primitives \citep{paraschos2013probabilistic}, smoothed- \citep{nagabandi2020deep,yang2020data} or coloured noise \cite{PinneriiCEM}.
We use Gaussian processes \citep{gpml} as action priors and show how they can be scaled to high-dimensional action spaces through factorization of the covariance.
Evaluating on simulated robotic systems, we reproduce prior results on policy search while transferring these ideas to \mpc, matching prior performance with respect to sample efficiency while ensuring smooth actuation.

\textbf{Contribution.}
First, we present a perspective of episodic inference-based control based on Gibbs posteriors.  
Using this view, we then present novel Monte Carlo variants that incorporate the approximate inference error due to importance sampling, simplifying the hyperparameter while providing regularization.  
Thirdly, we demonstrate how richer Gaussian process priors can be combined with these regularized Gibbs posteriors for Monte Carlo \mpc using online sequential inference, which achieves greater smoothness and sample efficiency than standard white noise priors.
We highlight connections between this approach to \mpc and effective prior approaches to episodic policy search.

\section{Monte Carlo Methods for Optimal Control}
\label{sec:soc}
This section outlines the problem setting and introduces variational optimization and posterior policy iteration methods.
We consider the standard (stochastic) optimal control setting in discrete-time, with states $\vs\in\sR^{d_s}$, actions $\va\in\sR^{d_a}$.
Optimization is framed as maximizing a reward $r{\,:\,}\sR^{d_s}{\,\times\,}\sR^{d_a}{\,\rightarrow\,}\sR$ under dynamics $p(\vs_{t+1}\mid\vs_t,\va_t)$ and initial state distribution $p(\vs_1)$,
\begin{align}
    \max_{\va_1,\dots,\va_T}\; \E\left[\textstyle\sum_{t=1}^T r(\vs_t,\va_t)\right]
    \quad
    \text{s.t.}
    \quad
    \vs_{t+1}\sim p(\cdot\mid\vs_t,\va_t),
    \quad
    \vs_1 \sim p(\vs_1).\label{eq:soc}
\end{align}
This work focuses on the episodic setting, where optimization is performed after evaluating the current solution over a finite-time horizon $T$.
We frequently use the episodic return $R$, where $R(\mS,\mA){\,=\,} \sum_{t=1}^T r(\vs_t, \va_t)$, using upper-case to denote sequences, e.g. $\mA{\,:=\,}\{\va_1,\dots,\va_{T}\}$.

\subsection{Variational Optimization with Gibbs Posteriors}
\label{sec:vi}
The optimization outlined above is amenable to gradient-based solvers such as stochastic differential dynamic programming \citep{5530971}. 
However, to aid optimization through exploration and regularization, we can consider optimizing a parametric \emph{belief} over action sequences $q{\,\in\,}\gQ$. 
The variational formulation (\Eqref{eq:ELBO}) generalizes Bayes' rule beyond optimizing likelihoods and resembles many learning algorithms
\cite{zellner1988optimal,khan2021bayesian}.
This work concerns optimizing an open-loop action sequence to maximize an episodic return. 
Bayesian inference of an action sequence from data, known as input estimation, can be performed using message passing of the appropriate probabilistic graphical model, capturing the sequential structure of the problem and necessary priors \citep{watson2021stochastic}.
If the measurement log-likelihood is replaced with the control objective, this inference computation can be shown to have precise dualities with dynamic programming-based optimal control \citep{i2corl}.
While this switch in objective provides a powerful suite of inference tools for efficient computation, it requires treating the control objective as a Markovian log-likelihood, which is not the case for episodic objectives. 
The Gibbs likelihood is a general treatment of the objective-as-likelihood (Definition \ref{def:gibbs})~\citep{guedj2019primer, JMLRAlquier16}.

\newpage

\begin{definition}
(Gibbs likelihoods and posteriors)
For a loss $f$ and prior $p(\vx)$, the Gibbs posterior $q_\alpha$ for parameter $\vx$ is derived by constructing the Gibbs likelihood $\exp(-\alpha\,f(\vx))$ from the loss,
\begin{align}
	q_\alpha(\vx) = \frac{1}{Z_\alpha}
	\exp(-\alpha\,f(\vx))
\,p(\vx),
	\quad
	Z_\alpha = \int \exp(-\alpha\,f(\vx))\,p(\vx)\, \rd\vx,
	\quad
	\alpha \geq 0.
	\label{eq:Gibbs}
\end{align}
This posterior minimizes the following objective
\begin{align}
    q_\alpha = \textstyle\argmin_{q\in\gQ} \E_{\vx\sim q(\cdot)}[f(\vx)] + \textstyle\frac{1}{\alpha}\KL[q(\vx)\mid\mid p(\vx)].
    \label{eq:ELBO}
\end{align}
This objective appears in PAC-Bayes methods \citep{guedj2019primer}, mirror descent methods~\citep{dai2016provable} and Bayesian inference as the evidence lower bound objective when $f(\vx)$ is a negative log-likelihood~\citep{JMLRAlquier16}.
\label{def:gibbs}
\end{definition}

Augmenting the variational optimization objective with prior regularization (\Eqref{eq:ELBO}), we obtain an expression of the optimal belief in the action sequence (\Eqref{eq:Gibbs}).
The parameter $\alpha$ has a range of meanings, depending on context. 
In PAC-Bayes it is the dataset size, in mirror descent it is an update step size and in risk-sensitive control it is the sensitivity \citep{rawlik2013stochastic,i2cacc}. 
Example \ref{ex:newton} in Appendix \ref{sec:discussion} examines a tractable linear-quadratic-Gaussian example of this update, demonstrating its relation to Newton-like optimization and highlighting the effect $\alpha$ has on the regularized update.

\subsection{Posterior Policy Iteration}
The optimal control problem (\Eqref{eq:soc}) is ambiguous regarding whether the action sequence or state-action trajectory is the optimization variable.
Applying the Gibbs posterior to the optimal control setting recovers Rawlik et al.'s posterior policy iteration \cite{rawlik2013stochastic}, which can be implemented using the joint distribution or policy. 
We consider the following joint state-action distribution, that factorizes in the following Markovian fashion
$p(\mS,\mA){\,=\,}p(\vs_1)\prod_{t=1}^T p(\vs_{t+1}\mid\vs_t,\va_t)\,p(\va_t\mid\vs_t)$.
Posterior policy iteration updates the state-action distribution through the policy, constructing a Gibbs likelihood from the reward, as the dynamics and initial state distribution are constant.
\begin{definition} \label{def:ppi}
(Posterior policy iterations (\textsc{ppi}) \citep{rawlik2013probabilistic})
As the initial distribution and dynamics are shared by the prior and posterior joint state-action distribution, the joint Gibbs posterior $q_\alpha(\mS,\mA) \propto \exp(\alpha R(\mS,\mA))\, p(\mS,\mA)$
can be alternatively expressed using the policy posterior update 
$q_\alpha(\mA\mid\mS) \propto \exp(\alpha R(\mS,\mA)) \, p(\mA\mid\mS)$.
\end{definition}
Using this update, the key decisions are choosing $p(\mA\mid\mS)$, $\alpha$ and the inference approximation.
If $p$ and $q_\alpha$ are Gaussian, then \ppi involves iterative refinement of the distribution.
In the Monte Carlo setting, $q_\alpha$ takes the form of an importance-weighted empirical distribution.
To apply iteratively, $p$ is updated using the M-projection, following the objective (\Eqref{eq:ELBO}), i.e. a weighted maximum likelihood fit of the policy parameters \cite{deisenroth2013survey}.
This approach is a stochastic approximate expectation maximization
(\textsc{saem}) method \citep{wirth2016model} and described fully in Algorithm \ref{alg:ppi} in the Appendix.
We argue a key aspect of \ppi methods is how to specify the inverse temperature $\alpha$ during optimization (Section \ref{sec:ppi}), as it has a strong influence on the posterior, which is important when fitting rich distributions such as Gaussian processes (Section~\ref{sec:ppi_mpc}) from samples.
Gaussian process action priors can be applied to several control settings, such as policy search and model predictive control (Section \ref{sec:exp_res}).

\section{Posterior Policy Constraints for Monte Carlo Optimization}
\label{sec:ppi}
The Gibbs posterior in Definition \ref{def:ppi} has been adopted widely in control, albeit from a range of different perspectives, such as Bayesian smoothing \citep{toussaint2009robot}, solutions to the Feynman-Kac equation~\citep{7442792}, maximum entropy~\citep{levine2018reinforcement}, mirror descent~\citep{Boots-RSS-19} and entropy-regularized reinforcement learning~\cite{Peters2010REPS}.
An open question is how best to set $\alpha$ for Monte Carlo optimization?
Relative entropy policy search (Definition \ref{def:ereps}), provides a principled and effective means of deriving $\alpha$ for stochastic optimization, using the constrained optimization view of entropy-regularized optimal control. 

\begin{definition} \label{def:ereps}
(Episodic relative entropy policy search (e\reps) \citep{deisenroth2013survey})
Maximize the expected return, subject to a hard \kl bound $\epsilon$ on the policy update,
\begin{align*}
    \textstyle\max_\vtheta
    \E_{
    \vs_{t+1}\sim p(\cdot|\vs_t,\va_t),
    \va_t\sim q_\vtheta(\cdot|\vs_t),
    \vs_1\sim p(\cdot)
    }
    [R(\vs_t, \va_t)]
    \quad
    \text{s.t.}
    \quad
    \KL[q_\vtheta(\mA|\mS) \mid\mid p(\mA|\mS)] \leq \epsilon.&
\end{align*}
The posterior policy takes the form
${q_\vtheta(\mA|\mS){\,\propto\,}\exp(\alpha R) \,p(\mA|\mS)}$, where $\alpha$ is derived from Lagrange multiplier calculated by minimizing the empirical dual $\gG(\cdot)$ using $N$ samples,
\begin{align*}
\textstyle\min_\alpha \gG(\alpha) = 
\textstyle\frac{\epsilon}{\alpha} + \frac{1}{\alpha}\log\int p(\mS,\mA)\exp(\alpha R(\mS,\mA))\,\rd\mS\,\rd\mA
\approx 
\frac{\epsilon}{\alpha} + \frac{1}{\alpha}\log\frac{1}{N}\sum_{n=1}^N\exp(\alpha R_n).
\end{align*}
\end{definition}
\newpage
While \reps is a principled approach to stochastic optimization, we posit two weaknesses: The hard \kl constraint $\epsilon$ is difficult to specify, as it depends on the optimization problem, distribution family and dimensionality. Secondly, the Monte Carlo approximation of the dual has no regularization and may poorly adhere to the \kl constraint without sufficient samples.
Therefore, we desire an alternative approach that resolves these two issues, capturing the Monte Carlo approximation error with a simpler hyperparameter. 
To tackle this problem, we interpret the \reps update as a pseudo-posterior, where the temperature is calculated using the \kl constraint. 
We make this interpretation concrete by reversing the objective and constraint, switching to an equality constraint for the expectation,
\begin{align*}
    \textstyle\min_\vtheta
    \KL[q_\vtheta(\mA\mid\mS) \mid\mid p(\mA\mid\mS)]
    \quad
    \text{s.t.}
    \quad
    \E_{
    \vs_{t+1}\sim p(\cdot|\vs_t,\va_t),\,
    \va_t\sim q_\vtheta(\cdot|\vs_t),\,
    \vs_1\sim p(\cdot)
    }[\sum_t r(\vs_t, \va_t)]
    = R^*.&
\end{align*}
This objective is a \emph{minimum relative entropy problem}~\cite{1056374}, which yields the same Gibbs posterior as e\reps
(Lemma 1, Appendix \ref{sec:discussion}).
With exact inference, a suitable prior and oracle knowledge of the maximum return, this program computes the optimal policy in a single step by setting $R^*$ to the optimal value. 
However, in this work, the expectation constraint requires self-normalized importance sampling (\snis) on sampled returns $R^{(n)}$ using samples from the current policy prior,
\begin{align*}
    \E_{
    \vs_{t+1}\sim p(\cdot|\vs_t,\va_t),\,
    \va_t\sim q_\vtheta(\cdot\mid\vs_t),\,
    \vs_1\sim p(\cdot)
	}
    [{\textstyle\sum_t r(\vs_t, \va_t)}]
    \approx
    {\textstyle\sum_n} w^{(n)}_{q/p} R^{(n)}
    =
    \frac{\sum_n R^{(n)} \exp(\alpha\,R^{(n)})}{\sum_n \exp(\alpha\,R^{(n)})}
    =
    R^*.
\end{align*}
Rather than specifying $R^*$ here, we identify that this estimator is fundamentally limited by inference accuracy. 
We capture this error by applying an \is-derived concentration inequality to this estimate (Theorem \ref{ref:lower_bound}) \citep{metelli2020importance}.
This lower bound can be used as an objective for optimizing $\alpha$, balancing policy improvement with approximate inference accuracy.

{
\begin{theorem} \label{ref:lower_bound}
(Importance sampling estimator concentration inequality (Theorem 2, \citep{metelli2020importance}))
Let $q$ and $p$ be two probability densities such that $q{\,\ll\,}p$ and
$d_2[q\mid\mid p] < + \infty$.
Let $\vx_1,\vx_2,\dots,\vx_N$ i.i.d. random variables sampled from $p$ and $f{\,:\,}\gX{\,\rightarrow\,}\sR$ be a bounded function $(||f||_\infty<+\infty)$.
Then, for any
$0 < \delta \leq 1$ and 
$N>0$ with probability at least $1{\,-\,}\delta$:
\vspace{-1em}
\begin{align}
    \E_{\vx\sim q(\cdot)}
    [f(\vx)]
    \geq
    {\textstyle\frac{1}{N}
    \textstyle\sum_{i=1}^N}
    w_{q/p}
    (\vx_i)
    f(\vx_i)
    -
    ||f||_\infty
    \sqrt{\frac{(1-\delta)
    d_2[q(\vx)\mid\mid p(\vx)]}{\delta\, N}}
    .\label{eq:lb}
\vspace{-1em}
\end{align}
\end{theorem}}
The divergence term $d_2[q||\,p]$ is the exponentiated R\'enyi-2 divergence, $\exp\sD_2[q||\,p]$. 
While this is tractable for the multivariate Gaussian, it is otherwise not available in closed form.
Fortunately, we can use the effective sample size (\ess) \citep{kong1992note} as an approximation, as $\hat{N}_\alpha{\,\approx\,}N \,/\,d_2[q_\alpha||\,p]$ \cite{metelli2020importance, cortes2010learning} (Lemma 2, see Section \ref{sec:discussion} of the Appendix).
Combining \Eqref{eq:lb} with our constraint, instead of setting $R^*$, we maximize the \textsc{is} lower bound $R_{\text{\textsc{lb}}}^*$ to form an objective for the inverse temperature $\alpha$ which incorporates the inference accuracy due to the sampling given inequality probability $1-\delta$,
\begin{align}
    \max_\alpha R_{\text{\textsc{lb}}}^*(\alpha, \delta) 
    &=
    \E_{q_\alpha/p}[R] -
    \gE_R(\delta, \hat{N}_\alpha),
    \quad
    \gE_R(\delta, \hat{N}_\alpha) = ||R||_\infty\sqrt{\frac{(1-\delta)}{\delta}}\frac{1}{\sqrt{\hat{N}_\alpha}}. 
\end{align}
We refer to this approach as \emph{lower-bound policy search} (\lbps).
This objective combines the expected performance of $q_\alpha$,  based on the \is estimate $\E_{q_\alpha/p}[\cdot]$, with regularization $\gE_R$ based on the return and inference accuracy.
Treating $p$, $N$, $||R||_\infty$ as task-specific hyperparameters, the only algorithm hyperparameter $\delta{\,\in\,}[0,1)$ defines the probability of the bound.
In practice, self-normalized importance sampling is used for \ppi, as the normalizing constants of the Gibbs likelihoods are not available. 
While Metelli et al. also derive an \snis lower bound \citep{metelli2020importance}, we found, as they did, that the \textsc{is} lower bound with \snis estimates work better in practice due to the conservatism of the \snis bound.
An interpretation of this approach is that the R\'enyi-2 regularization constrains the Gibbs posterior to be one that can be estimated from the finite samples, as the divergence is used in evaluating \is sample complexity ~\citep{agapiou2017importance, hernandez2020robust}. 
Moreover, the role of the \ess for regularization is similar to the `elite' samples in \cem.
Connecting these two mechanisms as robust maximum estimators (Section \ref{sec:discussion}), we also propose \emph{effective sample size policy search} (\essps), which optimizes $\alpha$ to achieve a desired \ess $N^*$, i.e. a R\'enyi-2 divergence bound, using the objective $\min_\alpha |\hat{N}_\alpha{-}N^*|$. 
More details regarding \ppi (Section \ref{sec:discussion})
and temperature selection methods
(Table \ref{tab:temperature}) are in the Appendix.

This section introduces two methods, \lbps and \essps, for constraining the Gibbs posteriors for Monte Carlo optimization. 
These methods provide statistical regularization through soft and hard constraints involving the effective sample size, which avoids the pitfall of fitting high-dimensional distributions to a few effective samples. 
A popular setting for these methods is \mpc, which performs episodic optimization over short planning horizons while adapting each time step to the current state. 
Moreover, for optimal control, we also need to specify a suitable prior over action sequences.
To apply \ppi to the \mpc setting, we must implement online optimization given this prior over actions.

\newpage

\begin{figure}[tb]
	\vspace{-3em}
	\includegraphics{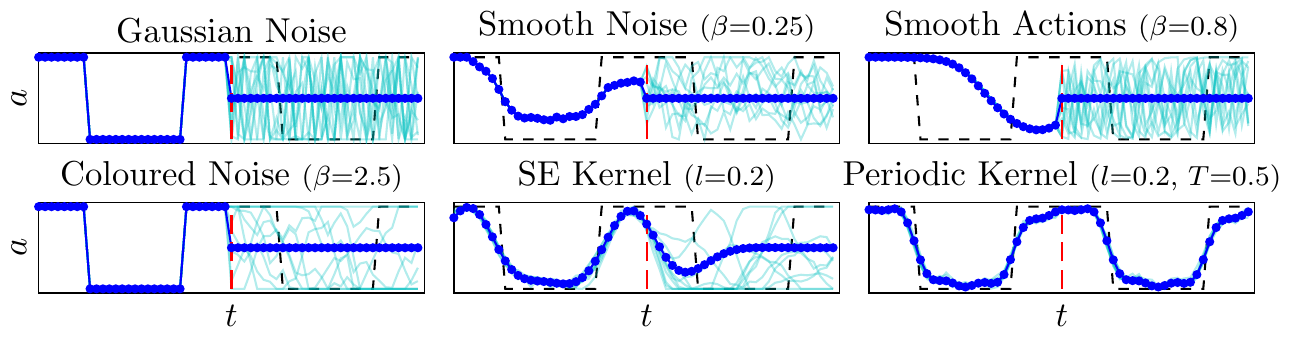}
	\vspace{-1em}
	\caption{A practical aspect of Monte Carlo control methods for robotics is optimizing smooth action sequence. 
	This example shows a non-smooth optimal sequence \blackdashed, which may be undesirable, though optimal, to fit exactly.
	Prior methods struggle at providing both effective smooth solutions in the mean \bluesolid and action samples \cyansolid, as they ultimately fit the action distribution in an independent fashion. Using kernel-derived covariance function provides both. The line \reddashed denotes the optimization horizon, beyond which is exploratory actions derived from both the posterior and prior. 
}
	\label{fig:policy_sample}
	\vspace{-0.5em}
\end{figure}
\section{Online Posterior Policy Iteration \& Prior Design}
\label{sec:ppi_mpc}
In this section, we derive the online realization of posterior policy iteration that uses and maintains correlated action priors, computing the finite-horizon $H$ future actions given a likelihood on a subset of actions from the past.
$R$ represents the return-based Gibbs likelihood term (Definition \ref{def:gibbs}),
\begin{align}
    q_\alpha(\va_{t:t+H}\mid R_{1:\tau})
    &=
    \textstyle\int q_\alpha(\va_{1:t+H}\mid R_{1:\tau})\,\rd\va_{1:t-1}  
    \propto
    \textstyle\int p(R_{1:\tau}\mid\va_{1:\tau})\,p(\va_{1:t+H})\,\rd\va_{1:t-1},
    \label{eq:update_marg}
\end{align}
where $\tau\leq t+H$. 
As an analogy, this is equivalent to combining forecasting with state estimation, i.e. $p(\vx_{t:t+H}|\vy_{1:t})$ for states $\vx$ and measurements $\vy$.
For correlated priors on the action space, this computation is tractable if working with Gaussian processes.
In fact, a recurring aspect across several many posterior policy iteration-like approaches is the use of Gaussian process policies,
\begin{align*}
    p(\mA\mid\mS) =
    \begin{cases}
    \prod_t\gN(\vmu_t, \mSigma_t),
    &\text{(Independent Gaussian noise, e.g. \citep{williams2017model}),}\\
    \gN(\vmu_\vw\tran\phi(t),  \phi(t)\tran\mSigma_\vw\phi(t)),
    &\text{(Bayesian linear regression, e.g. \promp \protect\citep{paraschos2013probabilistic})},\\
	\prod_t\gN(\vk_t + \mK_t \vs, \mSigma_t),
	&\text{(time-varying linear Gaussian e.g. \citep{rawlik2013stochastic,gomez2014policy, i2cacc}),}\\
    \gG\gP(\vmu(\vs), \mSigma(\vs)),
    &\text{(non-parametric Gaussian process \cite{np-reps}).}\\
    \end{cases}
\end{align*}
Despite the simplicity of Gaussian action noise, for robotics, more sophisticated noise is often desired for safety and effective exploration \cite{NIPS2008KoberPeters_54110,deisenroth2013survey}.
Prior work has proposed first-order smoothing \citep{nagabandi2020deep,yang2020data}. 
Using $\vv_t^{(n)} {\,\sim\,} \gN(\vzero, \mI)$, $\beta\in[0,1]$ and $\mSigma_t=\mL_t\mL_t\tran$, actions are sampled using
\begin{align*}
    \va^{(n)}_t &= \vmu_t + \mL_t \vn_t^{(n)},
	\quad
	\vn^{(n)}_t = \beta \vv_t^{(n)} + (1-\beta) \vn^{(n)}_{t-1},
	\;\;\text{or}\;\;
	\vn^{(n)}_t = \beta \vv_t^{(n)} + \sqrt{(1-\beta^2)} \,\vn^{(n)}_{t-1}.
\end{align*}
However, in practice it is also implemented as
$
\va^{(n)}_t = \beta (\vmu_t + \mL_t \vv_t^{(n)}) + (1-\beta)\va^{(n)}_{t-1}$
\footnote{See the source code for Nagabandi et al. \cite{nagabandi2020deep}  and \texttt{MBRL-lib} \cite{Pineda2021MBRL}.}.
While this approach directly smooths the actuation, it also introduces a lag, which may deteriorate performance.
Other approaches have used colored noise for sampling the noise $\vn$ \citep{PinneriiCEM}.
Contrast these approaches to Gibbs sampling a multivariate Gaussian joint distribution with 1-step cross-correlations~\cite{doucet2010note},
which is
$\va^{(n)}_{t|t-1} {\,=\,} \vmu_{t|t-1}^{(n)} + \mL_{t|t-1}\vv_t^{(n)}$,
where
$
    {\vmu_{t|t-1}^{(n)} = \vmu_t + \mSigma_{t,t-1}\mSigma_{t-1}\inv(\va^{(n)}_{t-1} - \vmu_{t-1})},
$
and
$
    {\mSigma_{t|t-1} = \mSigma_t - \mSigma_{t,t-1}\mSigma_{t-1}\inv\mSigma_{t,t-1}\tran}.
$
The differences are subtle, but important. 
The initial proposed sampling scheme essentially adds correlated noise to the mean for exploration, but does not consider the smoothness of the mean itself.
The practical implementation incorporates the previous action, but through exponential smoothing, which introduces a fixed lag that potentially degrades the quality of the mean action sequence. 
Correct sampling of the joint distribution has neither of these issues and naturally extends to correlations over several time steps.
We do this in a general fashion by considering the (continuous time) Gaussian process (see Section \ref{sec:stoch_processes}, Appendix), so
$p(\va_\vt) = 
\gN(\vmu_{t_i:t_j}, \mSigma_{t_i:t_j}) = 
\gG\gP(\vmu(\vt), \mSigma(\vt))$ for a discrete-time sequence $\vt = [t_i,\dots,t_j]$. Proposition \ref{prop:gp} shows how the time shift for \mpc can be implemented in a general fashion when using~\gp{}s.
\nolinebreak
\begin{restatable}{prop}{gpshift}
\label{prop:gp}
Given a Gaussian process prior $\gG\gP(\vmu(t), \mSigma(t))$ and multivariate normal posterior $q_\alpha(\va_{t_1:t_2}){\,=\,}\gN(\vmu_{t_1:t_2|R},\mSigma_{t_1:t_2|R})$ for $t_1$ to $t_2$, the posterior for $t_3$ to $t_4$ is expressed as
{
\setlength{\abovedisplayskip}{0pt}
\setlength{\belowdisplayskip}{0pt}
\setlength{\abovedisplayshortskip}{0pt}
\setlength{\belowdisplayshortskip}{0pt}
\begin{equation}
	\vmu_{t_3:t_4|R} {\,=\,} \vmu_{t_3:t_4}{+\,}\mSigma_{t_3:t_4,t_1:t_2}
	\vnu_{t_1:t_2},\;\;\;
	\mSigma_{t_3:t_4|R}{\,=\,}\mSigma_{t_3:t_4}{-\,} \mSigma_{t_3:t_4,t_1:t_2}
	\mLambda_{t_1:t_2}
	\mSigma_{t_3:t_4,t_1:t_2}\tran{,}\hspace{-0.5em}
	\label{eq:cov_update}
\end{equation}
}
where  $\vnu_{t_1:t_2}{=} 
\mSigma_{t_1:t_2}\inv(\vmu_{t_1:t_2|R}{-}\vmu_{t_1:t_2})$
and 
$\mLambda_{t_1:t_2}{=} 
\mSigma_{t_1:t_2}\inv(\mSigma_{t_1:t_2}{-} \mSigma_{t_1:t_2|R})\mSigma_{t_1:t_2}\inv$.
\end{restatable}
This update combines the new sequence prior from $t_3$ to $t_4$ and the previous likelihood used in the update for $t_1$ to $t_2$, obtained from the posterior and prior.
Note, the cross-covariance $\Sigma_{t_3:t_4,t_1:t_2}$ is computed using the covariance function of the prior \gp.
The proof is in Appendix~\ref{sec:discussion}.
\newpage

For a stationary kernel, fixed planning horizon and fixed control frequency, the term $\mSigma_{t_1:t_2}\inv$ is $\mSigma_{t:t+H}\inv$ and is constant, so can be computed at initialization to avoid repeated inversion.
Figure \ref{fig:policy_sample} demonstrates how this update lets us combine our prior with previous posterior in a principled fashion.
Moreover, its continuous-time construction means that the time resolution can be updated, not just the time window, for planning at different timescales \citep{Mukadam_mp}. 

Compared to the independence assumption, modeling correlations between actions introduces complexity. 
The full covariance over (flattened) time and action has a complexity $R(T^3d_a^3)$, which is infeasible to work with.
Assuming independence between actions, a \gp per action has a complexity of $R(T^3 d_a)$, requiring $d_a$ \gp{}s to be fit, which is not desirable for online methods such as \mpc.
To avoid the linear scaling w.r.t. $d_a$, we propose using the \emph{matrix Normal distribution} (Definition \ref{def:mavn}) for scalability, as it is parameterized into single $T$ and $d_a$-dimensional covariances, 
\begin{definition}
	(Matrix Normal Distribution (MaVN) \citep{dawid1981some}) 
	For a random matrix $\mX \in \sR^{n\times p}$,
	it follows a matrix normal distribution
	$\mX{\,\sim\,}\gM\gN(\mM, \mK, \mSigma)$,
	where
	$\mM\in\sR^{n\times p}, \mK\in S^n_+$ and $\mSigma\in S^p_+$, 
	if and only if
	$\text{vec}(\mX){\,\sim\,}\gN(\text{vec}(\mM), \mSigma\otimes\mK)$, where $\otimes$ denotes the Kronecker product.
	\vspace{-0.5em} 
	\label{def:mavn}
\end{definition}
Using the Kronecker-structured covariance provides a useful decomposition of the time-based covariance $\mK$, that defines correlations between time steps, and an action covariance $\mSigma$ that captures correlations between actions. 
Typically we assume actions are independent, but cross-correlations could be learned from experience for richer coordination.
While this Kronecker structure does not fully capture the correlations between time and actions, the structure is very useful for \mpc on robotic systems, where the actions space could be very high but the planning horizon is sufficiently small for covariance estimation using a reasonable number of Monte Carlo rollouts.

\textbf{Feature Approximations.} Despite the matrix Normal factorization, computing the correlations between actions still requires a dense $H{\times}H$ covariance matrix $\mK$ for planning horizon $H$. 
To sparsify this quantity, we consider kernel approximations, such as the canonical basis functions $\sum_n k(\cdot,\vx_n)$ and \emph{spectral} approximations using random features $\sum_n\phi_n(\cdot)$ \cite{rahimi2007random}, for a Bayesian linear model $\bm{\phi}_\vt\tran\mW$.
Focusing on the squared exponential (\se) kernel, this results in radial basis function (\rbf) and random Fourier features (\rff) respectively.
Interestingly, \rbf features are closely related to probabilistic movement primitives, used extensively in policy search for robotics \citep{paraschos2013probabilistic}.  
For one-dimensional inputs, \rff{}s are effectively approximated by applying Gauss-Hermite quadrature \citep{hildebrand1987introduction} to the random weights~\citep{mutny2018efficient}. 
\rbf features and \rff{}s approximate w.r.t. time and frequency respectively and could be combined \citep{wilson2020efficiently}.
Using these continuous-time features, the optimization is now abstracted from planning horizon and control frequency, providing much greater flexibility. 
Secondly, due to the features, a factorized weight covariance approximation does not sacrifice smoothness.
Moreover, the moment updates described above are not needed, as only
$\bm{\phi}_{\vt}$ is updated.

\section{Related Work}\label{sec:relatedwork}

\textbf{Inference-based control.}
Posterior policy iteration was proposed by Rawlik et al. \cite{rawlik2013stochastic} and covers prior methods developed from Bayesian smoothing \citep{toussaint2009robot, i2corl}, expectation maximization \citep{Peters_PICML_2007, NIPS2008KoberPeters_54110}, entropy regularization \citep{Peters2010REPS,daniel2016hierarchical} and path integral \citep{theodorou2010generalized} perspectives.
For \mpc specifically, the path integral-based \mppi was proposed \cite{williams2017model}, with alternative formulations based on mirror descent \citep{Boots-RSS-19} and variational
inference \citep{okada2020variational, lambert2020stein, wang2021variational}.
Mukadam et al. \cite{MukadamCYB17} models the optimal state-action distribution as a sparse Gaussian process and uses linearization for approximate inference.
The same approach is used for Gaussian process motion planning \cite{Mukadam_mp}, which are also optimized using sampling \citep{stomp}.
Gaussian quadrature is also used for inference-based \mpc \citep{watson2021stochastic}.
Concurrent work uses the \ess for a temperature adjusting heuristic for \mppi \cite{carius2021constrained} and also combines policy search with \mpc using \ppi techniques \cite{song2022policy}.
See Section \ref{sec:extended_related_work} for a more in-depth discussion on these related works.

\textbf{Policy design and regularization.}
Smooth actuation is important in robot learning for safety and exploration, having been proposed for Monte Carlo \mpc \citep{nagabandi2020deep, yang2020data, PinneriiCEM} and more broadly incorporated using augmented objectives, parameter sampling and policy design, e.g.~\citep{caps2021, pmlr-v164-raffin22a, korenkevych2019autoregressive}.

\textbf{Stochastic search.} 
Probabilistic interpretations of black-box optimization algorithms are well established \citep{stulp2012path, ollivier2017information, abdolmaleki2017deriving}, however prior work did not connect the \ess and elite samples. 
\cem and extensions have also been adopted widely as a solver for \mpc \citep{chua2018deep,PinneriiCEM,lutter2021learning}.

\textbf{Gaussian processes for control.} This work adopts \gp{}s for correlated action priors. 
This is distinct from prior work which uses \gp{}s to approximate dynamics or value functions, e.g.~\citep{berkenkamp2015safe, 1383790, 5531,MAIWORM2018455}. 
\newpage

\section{Experimental Results}
\label{sec:exp_res}
We assess the Gibbs posterior methods and policy design empirically across various settings.
Black-box optimization (Section \ref{sec:bbo}) considers standard benchmarks, while policy search (Section \ref{sec:ps}) optimizes action sequences for a robotic task.
For \mpc\,(Section \ref{sec:mpc_main}), we evaluate online \ppi approaches with white noise and smooth priors on high-dimensional, contact-rich tasks.
For the code, see
\href{https://github.com/JoeMWatson/monte-carlo-posterior-policy-iteration}{\ttfamily github.com/JoeMWatson/monte-carlo-posterior-policy-iteration}.

\begin{figure}[!t]
	\vspace{-3em}
	\begin{minipage}{\textwidth}
		\centering
		\begin{tikzpicture}
	
	\begin{axis}[
		hide axis,
		height=2cm,
		xmin=10, xmax=50,
		ymin=0, ymax=1.0,
		legend cell align={center},
		legend columns=5,
		legend style={/tikz/every even column/.append style={column sep=0.3cm}, draw=none},
		]
		]
		
		\addlegendimage{semithick, red, mark=*, mark size = 1}
		\addlegendentry{\cem};
		\addlegendimage{semithick, darkturquoise0191191, mark=*, mark size = 1}
		\addlegendentry{\lbps};
		\addlegendimage{semithick, blue, mark=*, mark size = 1}
		\addlegendentry{\essps};
		\addlegendimage{semithick, goldenrod1911910, mark=*, mark size = 1}
		\addlegendentry{\mppi};
		\addlegendimage{semithick, green01270, mark=*, mark size = 1}
		\addlegendentry{\pitwo};
	\end{axis}
	
\end{tikzpicture}
		\vspace{-1em}
	\end{minipage}
	\includegraphics{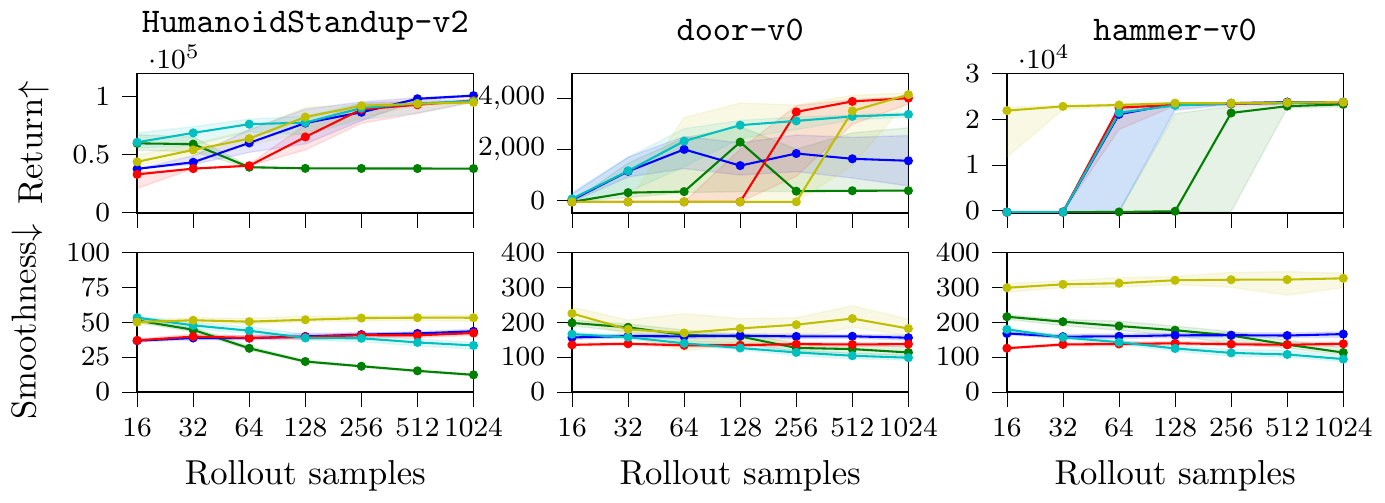}
	\vspace{-1.5em}
	\caption{
	\mpc return and \protect\hyperlink{smooth}{smoothness} with white noise priors. Displaying quartiles over 50 seeds. These priors require a large number of samples for good performance, across methods and tasks.
	}
	\label{fig:wn_mpc}
	\vspace{-0.8em}
\end{figure}

\vspace{-0.5em}
\subsection{Black-box Optimization} \label{sec:bbo}
\vspace{-0.3em}
To understand the behaviour of the proposed \ppi variants, the performance of \lbps and \essps on a range of standard black-box optimization functions over a range of hyperparameters, with e\reps and \cem as baselines, are shown in Appendix \ref{sec:bbo_res}. 
Figures \ref{fig:bbo1} -- \ref{fig:bbo4} show that \ess is a useful metric for these methods, as each solver exhibits consistent \ess for a given hyperparameter value.  
However, the uniform weights used by \cem (Figure \ref{fig:bbo1}) maintain entropy longer than \essps (Figure \ref{fig:bbo_essps}), which can lead to better optima, so the \ess is not sufficient to fully capture the behavior of these solvers.

\vspace{-0.5em}
\subsection{Policy Search} \label{sec:ps}
\vspace{-0.3em}

As \lbps and \essps are closely related to e\reps, we repeat the experiment from prior work performing the `ball in a cup' task using policy search using a Barret \textsc{wam} \citep{lutter2020differentiable}, which has been shown to transfer to the physical system \citep{klink2021probabilistic,lutter2020differentiable,muratore2021data}, as a benchmark task. 
Moreover, we replace \promp{}s with Matrix normal \rbf and \rff policies.
From the kernel perspective, this feature approximation is motivated by the large ($T\simeq 1000$) task horizon.
The results in Appendix \ref{sec:ps_res} confirm that these solvers are all capable of solving the task, based on success rate, where \rbf (Figure \ref{fig:ps_rbf}) and \rff{} features (Figure \ref{fig:ps_rff}) perform equally well w.r.t. the convergence of the success rate for each approach.

\vspace{-0.5em}
\subsection{Model Predictive Control with Oracle Dynamics} \label{sec:mpc_main}
\vspace{-0.3em}
\label{sec:mpc}
We evaluate online \ppi across a range of high-dimensional robotic control tasks in \texttt{MuJoCo} \citep{mujoco}, 
including \hsu in \texttt{Gym} \citep{gym} and
\door, \hammer from \texttt{mj\_envs}, using the Adroit hand (Figure \ref{fig:pull})~\cite{Rajeswaran-RSS-18}. 
\hypertarget{smooth}{
To measure smoothness, we adopt the \textsc{fft}-based score $\frac{2}{N f_s}\sum_{i=1}^N a_i f_i$  \citep{caps2021}, with sampling frequency $f_s$ and $N$ resolvable frequencies $f$ with amplitudes $a$. 
We compute the Euclidean norm of the action sequence over time and apply the smoothness measure to this signal.
}
For the evaluation, we focus on a low computational budget, with 1 or 2 iterations per timestep. 
To assess the impact of approximate inference, we assess performance over an logarithmic range of sample rollouts, following prior work \citep{PinneriiCEM}.
Details may be found in Appendix \ref{sec:experiments}. 

\textbf{White noise priors.}
Figure \ref{fig:wn_mpc} shows \mpc with white noise priors using \lbps and \essps, with \mppi, \cem and \pitwo baselines (see Table \ref{tab:temperature}). 
While each solver performs comparably for 1024 rollouts, the low sample regime shows greater performance variance. 
While \mppi seems particularly effective, Figure \ref{fig:ess_wn} shows that its average \ess is particularly low, ${\,\simeq\,} 1$ for many cases.
Combined with the fixed variances, this suggests optimization is closer to greedy random search than importance sampling.  
The poor \texttt{door-v0} performance of \essps is due to slow opening, rather than task failure.

\textbf{Policy design for smooth control.} Figure \ref{fig:smooth_mpc} shows online \ppi with action priors. \lbps, \essps and \mppi use the \se kernel, with \icem \citep{PinneriiCEM} and \mppi with smooth actions and noise as a baseline. 
The smooth action distributions greatly improves performance across models, tasks and sample sizes, due to effective exploration.  
As desired, the richer \se kernel provides much greater smoothness, by up to a factor of 2 compared to baselines, with limited impact to performance.
It is unsurprising that the smoothness bias reduces performance if optimal behavior is non-smooth, as illustrated in Figure \ref{fig:policy_sample}. 
Appendix \ref{sec:profiles} shows some of the action sequences from Figure \ref{fig:smooth_mpc}, where we see \gp smoothness varies with increasing rollout samples and also results in significant actuator amplitude reduction.

\textbf{Comparing kernel- and feature-based policies.}
To assess feature approximations for smooth \mpc, we replace the \se kernel with \rbf and \rff features, while keeping the lengthscale fixed. 
These policies perform worse given fewer samples, but are comparable to the true kernel with sufficient samples (Figure \ref{fig:mpc_ret_feat}).
We attribute this to the compounding errors of kernel approximation and fewer effective samples.
In contrast to the policy search task, \rff{}s appear superior to \rbf features.

\begin{figure}[!bt]
	\vspace{-1em}
	\begin{minipage}{\textwidth}
		\centering
		\begin{tikzpicture}
	
	\begin{axis}[
		height=2cm,
		hide axis,
		xmin=10, xmax=50,
		ymin=0, ymax=1.0,
		legend cell align={center},
		legend columns=3,
		legend style={/tikz/every even column/.append style={column sep=0.3cm}, draw=none},
		]
		]
		
		\addlegendimage{semithick, darkviolet1910191, mark=*, mark size = 1}
		\addlegendentry{\icem (coloured noise)};
		\addlegendimage{semithick, darkturquoise0191191, mark=*, mark size = 1}
		\addlegendentry{\lbps (\se kernel)};
		\addlegendimage{semithick, blue, mark=*, mark size = 1}
		\addlegendentry{\essps (\se kernel)};
		\addlegendimage{semithick, goldenrod1911910, mark=*, dashed, mark size = 1}
		\addlegendentry{\mppi (smooth actions)};
		\addlegendimage{semithick, goldenrod1911910, mark=*, dash pattern=on 1pt off 3pt on 3pt off 3pt, mark size = 1}
		\addlegendentry{\mppi (smooth noise)};
		\addlegendimage{semithick, goldenrod1911910, mark=*, mark size = 1}
		\addlegendentry{\mppi (\se kernel)};
	\end{axis}
	
\end{tikzpicture}
		\vspace{-1em}
	\end{minipage}
	\includegraphics{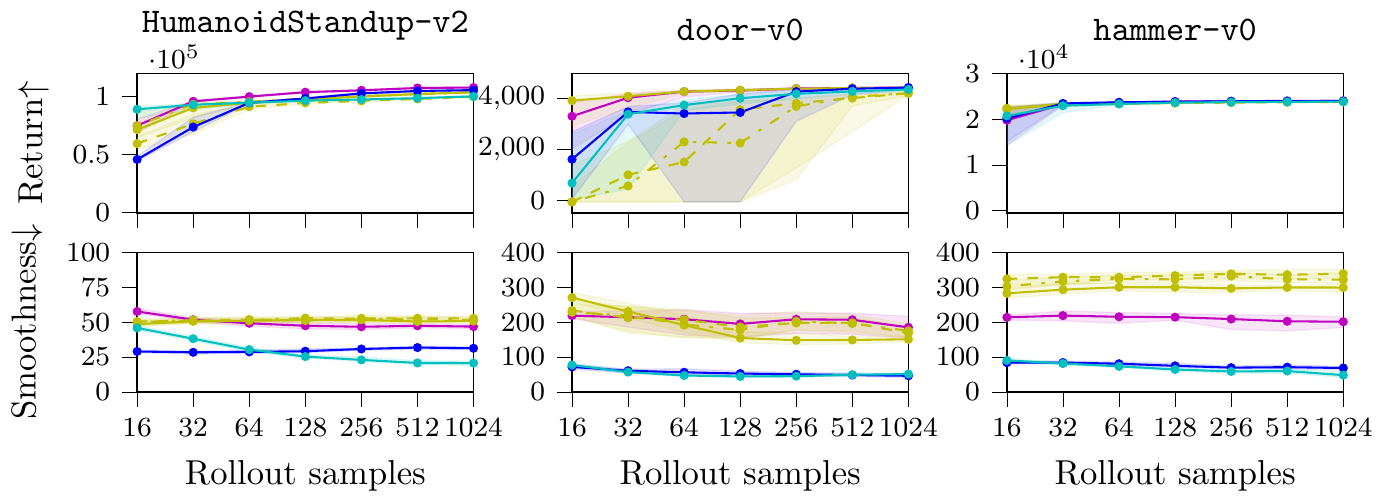}
	\vspace{-1.5em}
	\caption{
	\mpc return and \protect\hyperlink{smooth}{smoothness} with smooth action priors. Displaying quartiles over 50 seeds.
	Compared to white noise priors, smooth action priors improve sample efficiency dramatically, but only \ppi methods (\lbps, \essps) preserve this smoothness in the downstream control.
	}
	\label{fig:smooth_mpc}
	\vspace{-1.7em}
\end{figure}

\textbf{Learning priors from data.}
A benefit of using \gp{} priors is the ability to optimize hyperparameters from expert demonstrations through the likelihood or a divergence.
Moreover, the matrix normal distribution is useful for analyzing high-dimensional action sequences, as it decomposes temporal- and action correlations into viewable covariance matrices.
Section \ref{sec:priors_from_data} shows the matrix normal distributions of expert demonstrations of the tasks, obtained through human and \rl experts.
The results show that, surprisingly, the demonstrations are rougher than the smoothness achievable with \mpc.  
We attribute this to control artifacts from demonstration collection and the use of Gaussian noise by \rl agents. 
Applying the same methodology to the demonstrations of the smooth \mpc agents proposed here extracts the expected action correlations across tasks.
This analysis also raises the question of whether smoothness is an inductive bias we enforce for practicality, or a phenomena we expect to arise from optimality.  
If the latter, it may be that the simulated environments or objectives considered are lacking components that encourage smoothness, such as energy efficiency. 

\vspace{-1em}
\section{Conclusion}
\vspace{-0.9em}

We present a broad perspective on episodic posterior policy iteration method for robotics and new methods for the Monte Carlo setting, based on regularizing the \is approximations.
By considering vector-valued Gaussian processes for action priors, we have demonstrated how sample-efficient \mpc can be performed as online inference and with greater control over actuator smoothness, connecting Monte Carlo \mpc to prior work on policy search. 
This approach was validated on a set of high-dimensional \mpc tasks closely matching baseline performance while achieving greater smoothness.

\textbf{Limitations.}
Much of the prior work is motivated by simplicity, minimizing hyperparameter tuning and numerical procedures such as matrix inversion \citep{theodorou2010generalized}.
In contrast, the contributions of this work introduces complexity, i.e. online temperature optimization and the use of dense covariance matrices in order to perform more sophisticated approximate inference.   
While this additional complexity has an impact on execution time (Table \ref{tab:time}, Appendix), we hope the sample-efficiency when combined with accelerations such as \textsc{gpu}-integration should produce real-time algorithms \citep{bhardwaj2022storm}.
\newpage

\newpage

\clearpage
\acknowledgments{
This work built on prior codebases developed by Johannes Silberbauer, Michael Lutter and Hany Abdulsamad.
The large-scale experiments and ablations were conducted on the Lichtenberg high performance computer of the TU Darmstadt.
The authors wish to thank Hany Abdulsamad, Boris Belousov, Michael Lutter, Fabio Muratore, Pascal Klink,  Georgia Chavalvris, Kay Hansel, Oleg Arenz and the anonymous conference reviewers for helpful feedback on earlier drafts. 
}

\bibliography{lib}  %

\newpage

\appendix

\section{Extended Discussion Points and Results}
\label{sec:discussion}
This section discusses specific details not covered at length in the main section. 

\textbf{Posterior policy iteration as optimization.}
For clarity, we can ground inference-based optimization clearly by considering Gaussian inference and quadratic optimization (Example \ref{ex:newton}).

\begin{example} \label{ex:newton}
	(Newton method via Gaussian inference).
	Consider minimizing a function $f(\vx)$ from a Gaussian prior $p(\vx) = \gN(\vmu_i, \mSigma_i)$.
	Considering a second-order Taylor expansion about the prior mean,
	$f(\vx)\approx f(\vmu_i) + \nabla_\vx f(\vmu_i)(\vx-\vmu_i) + \frac{1}{2}(\vx-\vmu_i)^\top \nabla_{\vx\vx}f(\vmu_i)(\vx-\vmu_i)$,
	the Gaussian pseudo-posterior $q_\alpha(\vx) = \gN(\vmu_{i+1}, \mSigma_{i+1}) \propto \exp(-\alpha f(\vx))\,p(\vx)$ is 
	\begin{align*}
		\vmu_{i+1} = \vmu_{i} - \alpha\mSigma_{i+1}\nabla_\vx f(\vmu_i),\;
		\mSigma_{i+1} = (\mSigma_{i} + \alpha\nabla_{\vx\vx}^2f(\vmu_i))\inv.
	\end{align*}
	As $\alpha{\,\rightarrow\,}0$, $q_\alpha{\,\rightarrow\,}p$, while
	as $\alpha\rightarrow\infty$, the mean update tends to the Newton step,
	\begin{align*}
	\vmu_{i+1} \rightarrow \vmu_{i} - \nabla_{\vx\vx}^2f(\vmu_i)\inv\nabla_\vx f(\vmu_i).
\end{align*}	
	Therefore, $\alpha$ acts as a learning rate and regularizer, with repeated inference performing regularized Newton decent. 
	This regularization is beneficial for non-convex optimization, as the Hessian may be negative definite.
\end{example}

\begin{figure}[!t]
	\includegraphics{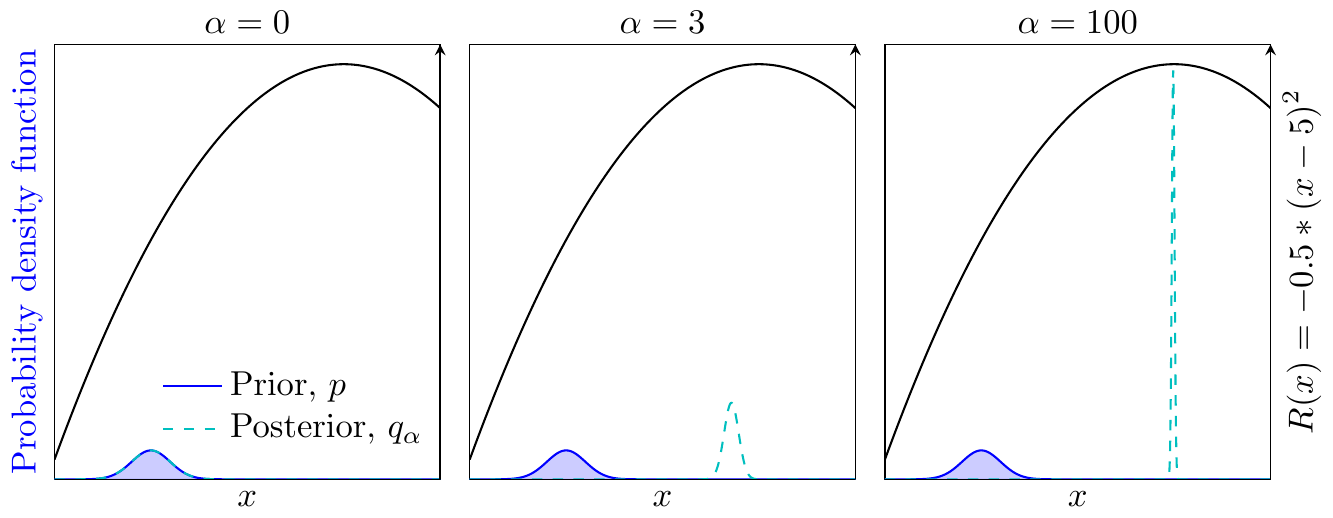}
	\vspace{-0.5em}
	\caption{
	Visualization of Example \ref{ex:newton} for three values of $\alpha$, showing interpolation between no update and towards the Newton update. 	
	}
	\label{fig:gaussian_ppi}
	\vspace{-1em}
\end{figure}

\begin{figure}[!b]
	\includegraphics{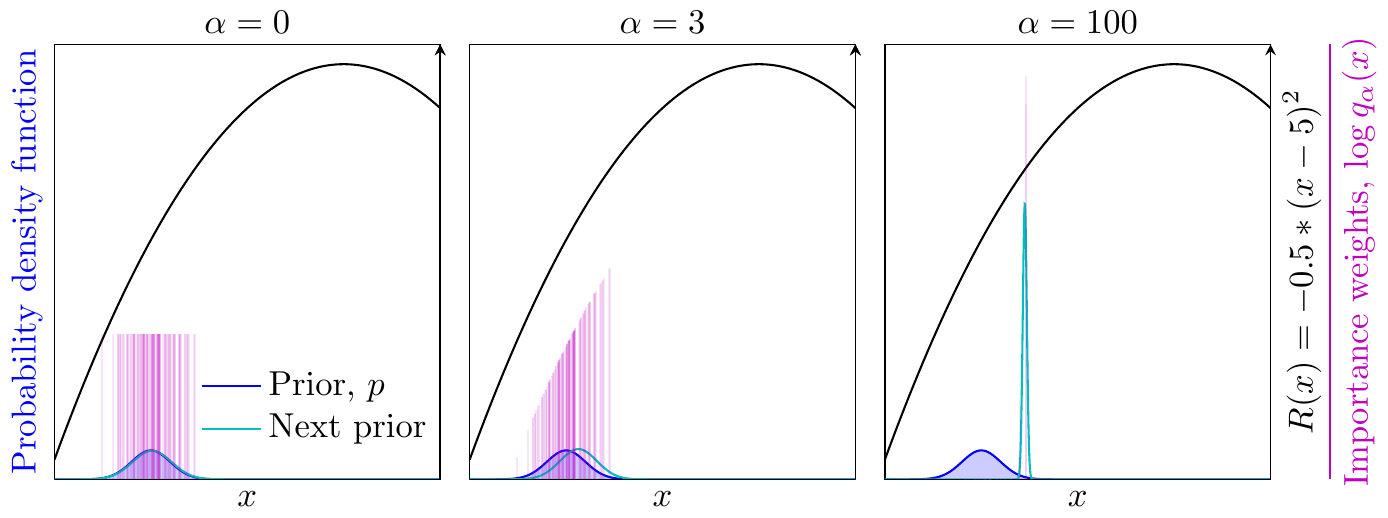}
	\vspace{-1.5em}
	\caption{
		Visualization of Example \ref{ex:newton} for three values of $\alpha$, using self-normalized importance sampling. 
		High values of $\alpha$ lead to the posterior collapsing around a few samples. Moreover, compared to the Gaussian case, the update is limited to the support of the samples.   	
	}
	\label{fig:gaussian_ppi}
	\vspace{-1em}
\end{figure}

\textbf{Minimum relative entropy problems.}
To motivate \reps as a form of constrained Gibbs posterior, we reverse the objective and constraint, which yields the minimum relative entropy problem (Lemma \ref{lemma:mrep}).
The posterior of this problem is the same as \reps and the temperature is also a Lagrangian multiplier, but the constraint is now defined by the expectation value rather than the \kl divergence.
\begin{lemma} \label{lemma:mrep}
	(Minimum relative entropy \cite{1056374})
	Find $q$ that minimizes the relative entropy from $p$, given a function $\vf$ that describes random variable $\vx$,
	$
	\textstyle\min_q \KL[q(\vx)\mid\mid p(\vx)]
	\;
	\text{s.t.}
	\;
	\E_{\vx\sim q(\cdot)}[\vf(\vx)] = \hat{\vf}.
	$
	The solution to this problem is
	$
	q(\vx) \propto \exp(-\bm{\lambda}\tran\vf(\vx))p(\vx),
	$
	where $\bm{\lambda}$ is the Lagrange multiplier required to satisfy the constraint. 
	When $p(\vx)$ is uniform, $q(\vx)$ is a maximum entropy distribution.
\end{lemma}
Interestingly, in maximum entropy problems, function $\vf$ defines the sufficient statistics of the random variable.
For our setting, $\vf$ is based on the reward function (episodic return), which indicates the reward objective is used to `summarize' the random policy samples.  

\textbf{R\'enyi-2 divergence and the effective sample size.}
The \lbps objective uses the effective sample size as a practical estimator for the exponentiated R\'enyi-2 divergence (Lemma \ref{lemma:ess}).
\begin{lemma} \label{lemma:ess}
	(Exponentiated R\'enyi-2 divergence estimator~\citep{cortes2010learning})
	The effective sample size diagnostic ($\hat{N}$) is an estimate of the exponentiated R\'enyi-2 divergence for finite samples, as ${d_2(p|| q)/N{\,\approx\,}\hat{N}^{-1}}$,
	\begin{align*}
		\lim_{N\rightarrow\infty}\frac{1}{N}\frac{(\sum_n w_{q/p}(\vx_n))^2}{\sum_n w_{q/p}(\vx_n)^2} = 
		\int \frac{p(\vx)}{q(\vx)}\rd\vx = \exp \sD_2(p\mid\mid q)^{-1}
		=
		d_2(p\mid\mid q)^{-1}.
	\end{align*}
\end{lemma}
This result connects divergence-based stochastic optimization methods, such as \reps, to `elite'-based stochastic solvers such as \cem, as the number of elites is equal to the effective sample size of the weights, as discussed below.

\textbf{Time Shifting with Gaussian Process}
We derive the multivariate normal posterior shifting update introduced in Section \ref{sec:ppi_mpc}, that is facilitated by continuous-time Gaussian process priors.
\gpshift*
\begin{proof}
For this update, we require the Gaussian likelihood $\gN(\vmu_{R|t_1:t_2}, \mSigma_{R|t_1:t_2})$ that achieves the posterior $q_\alpha(\va_{t_1:t_2})$ given the prior $p(\va_{t_1:t_2})$.
Using $\mK_{t_1:t_2}{\,=\,}\mSigma_{t_1:t_2}
\mSigma_{R|t_1:t_2}\inv$, we write the Gaussian posterior update in the Kalman filter form
\begin{align*}
	\vmu_{t_1:t_2|R} &= \vmu_{t_1:t_2} + \mK_{t_1:t_2}(\vmu_{R|t_1:t_2}-\vmu_{t_1:t_2})
	,\\
	\mSigma_{t_1:t_2|R} &= \mSigma_{t_1:t_2}- \mK_{t_1:t_2}\mSigma_{R|t_1:t_2}\mK_{t_1:t_2}\tran.
\end{align*}
We introduce $\vnu$ and $\mLambda$ to capture the unknown terms involving $\va_{R|t_1:t_2}$ w.r.t. $\va_{t_1:t_2}$ and $\va_{t_1:t_2|R}$, 
\begin{align}
	\hspace{-1em}
	\vnu_{t_1:t_2} &= \mSigma_{R|t_1:t_2}\inv(\vmu_{R|t_1:t_2}-\vmu_{t_1:t_2}) = 
	\mSigma_{t_1:t_2}\inv(\vmu_{t_1:t_2|R} - \vmu_{t_1:t_2})
	,\notag\\
	\mLambda_{t_1:t_2} &=
	\mSigma_{R|t_1:t_2}\inv = 
	\mSigma_{t_1:t_2}\inv(\mSigma_{t_1:t_2} - \mSigma_{t_1:t_2|R})\mSigma_{t_1:t_2}\inv
	.\notag
	\intertext{When extrapolating to the new time sequence, we use the Kalman posterior update again, but with the prior for the new time sequence.
	This step computes the joint over old and new timesteps, conditions on the objective (i.e. Bayes' rule), and marginalizes out the old timesteps (see Equation \protect\ref{eq:update_marg}), so}
	\vmu_{t_3:t_4|R} &= \vmu_{t_3:t_4}{+\,}\mSigma_{t_3:t_4,t_1:t_2}
	\vnu_{t_1:t_2},\notag\\
	\mSigma_{t_3:t_4|R} &= \mSigma_{t_3:t_4}{-\,} \mSigma_{t_3:t_4,t_1:t_2}
	\mLambda_{t_1:t_2}
	\mSigma_{t_3:t_4,t_1:t_2}\tran{.}\hspace{-0.5em}
	\notag
\end{align}
\end{proof}
Using the geometric interpretation of the Kalman filter \cite{anderson2012optimal}, this update can be seen as projecting the solution from the previous sequence into the new time sequence, given the correlation structure defined by the \gp prior.

\textbf{Augmented Cost Design versus Prior Design.}
Prior work enforces smoothness using an augmented reward function term, e.g. $-\lambda_s||\va_t-\va_{t+1}||^2$ per timestep \cite{caps2021}. 
For \ppi, the \ppi objective also augments the reward objective with a \kl term (Equation \ref{eq:ELBO}).
For the Monte Carlo setting here, this \kl term is $\frac{1}{\alpha}\textstyle\sum_n w^{(n)} (\log w^{(n)} - \log p(\mA^{(n)}))$, where $w^{(n)}{\propto}\exp(\alpha R^{(n)})$.
When $p(\mA)$ is a Gaussian process, the log probability has the familiar quadratic form, but this time applies to the whole sequence $\mA$, due to the episodic setting.  
Independent white noise only applies a quadratic penalty per action, due to its factorized covariance.
A first-order Markovian \gp would regularize one-step correlations, as its inverse covariance is banded \cite{barfoot2014batch}. 
The squared exponential kernel, used in this work, has infinite order \cite{gpml} and a dense inverse covariance matrix and therefore regularizes the whole sequence, which is why it was chosen to achieve smoothness.

\textbf{Episodic or sequential inference?}
For the Monte Carlo setting, sequential Monte Carlo (SMC)~\citep{cappe2007overview} is an obvious choice for inference given its connection to Bayesian smoothing methods.
However, we found that an SMC approach provided limited benefits for control. 
The variance reduction when resampling decreases exploration from the optimization viewpoint.

Moreover, SMC smoothing reweighs the particles sampled in the forward pass, with $\gO(N^2T)$ complexity for $N$ particles \citep{cappe2007overview}.
Therefore, no exploration occurs during the backward pass, in contrast to Gaussian message passing where smoothing performs Riccati equation updates on the state-action distribution. 
When considering just sequential importance sampling for trajectories, we arrive at the episodic case, as 
$w^{(n)}_T = \exp\left(\alpha r^{(n)}_T\right) w_{T-1}^{(n)} = \exp\left(\alpha \sum_{\tau=0}^T r^{(n)}_\tau\right)$. 

\textbf{Connection to rare event simulation and maximum estimators.}
The cross entropy method is a popular and effective Monte Carlo optimizer \citep{cem}.
It broadly works by moment matching an exponential family distribution onto `elite' samples.
The elite samples are chosen according to 
${\sP_{q(\vx)}(f(\vx) > f^*)}$, a rare event simulator, where improving upon $f^*$ is treated as the rare event.
In practice, elites are chosen by sorting the top $k$ samples according to the objective.
Our lower-bound reflects \cem in two ways. 
One, the rare event inequality reflects the expected return bound used in the \lbps objective.
Secondly, the $k$ in the top-$k$ estimator matches the effective sample size, used in \lbps and \essps. 
One can view \cem and \is approaches through their maximum estimators
\begin{align*}
\text{Max}_k[\mR] = \frac{\sum_{n=1}^k R_n}{\sum_{n=1}^k 1},
\quad
\text{Max}_\alpha[\mR] = \frac{\sum_{n=1}^N \exp(\alpha R_n)R_n}{\sum_{n=1}^k \exp(\alpha R_n)}.
\end{align*}
We can see that the \snis expectation is equivalent to the Boltzmann softmax, used in reinforcement learning \citep{asadi2017alternative}. 
Using $k$ or $\alpha$ this estimator is bounded between the true maximum over the samples and the mean. 
Therefore the estimators behave in a similar fashion when $\alpha$ is chosen such that the \ess is $k$.
The key distinction is the sparse and uniform weights of \cem vs. the posterior weights of \ppi methods, which seems to have an equal or more important influence on optimization, based on the black-box optimization results.

\textbf{Connection to mirror descent.}
Functional mirror descent for online learning optimizes a distribution $q$ and introduces a Bregman divergence $\sD$ for information-geometric regulation, i.e.
$q_{i+1}=\argmin_{q\in\sQ} \E_{\vx\sim q(\cdot)}[f(\vx)] + \alpha_i\sD[q\mid\mid q_{i}]$ \cite{pmlr-v97-chu19a}.
The temperature typically follows a pre-defined schedule that comes with convergence guarantees, e.g. $\alpha_i{\,\propto\,} \sqrt{i}$, such that $\alpha{\,\rightarrow\,}\infty$ as $i{\,\rightarrow\,}\infty$~\cite{beck2003mirror}.
Considering Equation \ref{eq:ELBO}, the \dmd view applies to our setting when using \kl regularization, so \dmd view opens up alternative Bregman divergences for regularization.

\section{Extended Related Work}
\label{sec:extended_related_work}
The methods presented in the main section have been investigated extensively in the prior literature. 
This section provides a more in-depth discussion of the differences in approach and implementation.

\textbf{Temperature tuning approaches.}
Across control-as-inference methods, the temperature value plays an important role. 
By the construction of the policy updates, mis-specification of the temperature leads overly greedy or conservative optimization.
Table \ref{tab:temp} provides a summary.
A popular and often effective approach is to tune a fixed temperature, as done in \textsc{aico} \cite{toussaint2009robot} and \mppi \citep{williams2017model}.
However, the optimization behavior depends the values of the objective, and nonlinear optimization methods such as Levenberg–Marquardt and mirror descent suggest an adaptive regularization scheme may perform better.
Reward-weighted regression (\textsc{rwr}) \citep{Peters_PICML_2007}, and later input inference for control (\textsc{i2c}) \citep{watson2021stochastic}, take a completely probabilistic view of the Gibbs likelihood and used a closed-form update to optimize the temperature using expectation maximization.
However, this update is motivated by the probabilistic interpretation of the optimization, not the optimization itself, so there is no guarantee this approach improves optimization, beyond the attractive closed-form update.
Follow-up work regulated optimization by normalizing the returns, using the range or standard deviation \citep{theodorou2010generalized, NIPS2008KoberPeters_54110}.
While this resolved the objective dependency, it is a heuristic and not interpretable from the optimization perspective. 
Later, \reps framed the update as a constrained optimization problem, subject to a \kl constraint \cite{Peters2010REPS, daniel2016hierarchical}.
The temperature was them obtained by minimizing the Lagrangian dual function.
However, in practice this dual is approximated using Monte Carlo integration which limits the constraint accuracy.
Moreover, a hard \kl constraint now requires a distribution-dependent hyperparameter. 
This work introduces \lbps and \essps.
\lbps optimizes a lower-bound of the importance-sampled expected return, with the hyperparameter controlling greediness via the confidence of the concentration inequality and therefore is objective- and distribution independent.
\essps bridges \lbps and \cem, optimizing for a desired effective sample size as an analogy to elite samples.
Finally, self-paced contextual episodic policy search has previously adopted the minimum \kl form of the \reps optimization problem, in order to assess the expected performance per context task \cite{klink2021probabilistic}.

\begin{table}[!tb]
	\centering
	\caption{
	A review of different adaptive temperature strategies across episodic control-as-inference methods and settings.
	We propose \lbps and \essps as optimization-driven methods with intuitive hyperparameters.
		}
	\label{tab:temp}
			\begin{tabular}{lll}
					Algorithm & Method & Description
					\\\toprule
					\mppi \cite{williams2017model}, \textsc{aico} \cite{toussaint2009robot} & Constant & $\alpha$\\
					\pitwo \cite{theodorou2010generalized}, \textsc{p}o\textsc{wer} \cite{NIPS2008KoberPeters_54110} & Normalization & $\bar{\alpha}/(\max[\mR] - \min[\mR]),\quad\bar{\alpha}/\sigma_R$ \\
					\textsc{rwr} \cite{Peters_PICML_2007}, \textsc{i2c} \cite{i2corl} & \textsc{em} & $\alpha_{i+1}=\sum_n \exp\left(\alpha_i R_i^{(n)}\right) / \sum_n R_i^{(n)}\exp\left(\alpha_i R_i^{(n)}\right) $ \\
					\reps \cite{Peters2010REPS}, \more \cite{abdolmaleki2015model} & \textsc{kl} bound & $\argmin_\alpha \epsilon/\alpha + \frac{1}{\alpha}\log\sum_n\exp(\alpha R^{(n)})$\\
					\lbps \tiny(this work) & \is lower-bound & $\argmax_\alpha \E_{q_\alpha}[R] - \gE_R(\delta, \hat{N}_\alpha)$ \\
					\essps \tiny(this work) & \ess & $\argmin_\alpha |\hat{N}_\alpha-N^*|$ \\
					\bottomrule     
				\end{tabular}
				\label{tab:temperature}
\end{table}

\textbf{Gaussian message passing methods.}
Using Bayesian smoothing of the state-action distribution, trajectory optimization can be performed by treating a step-based objective analogously to the observation log likelihood in state estimation \cite{toussaint2009robot,watson2021stochastic}.
While the control prior defined as temporally independent in prior work, the smoothing of the state-action distribution imbues an inherent smoothness to the solution (e.g. Figure 6, \citep{watson2021stochastic}).
Approximate inference can be performed using linearization or quadrature, which, while effective are less amenable to parallelization and therefore do not scale so gracefully to high-dimensional state and action spaces \cite{watson2021stochastic}.
Moreover, these methods require access to the objective function for linearization, rather than just rollout evaluations. 

\textbf{Mirror descent model predictive control.}
Mirror descent, used for proximal optimization, introducing information geometric regularization through a chosen Bregman divergence $\sD_\Psi$.
\emph{Dynamic} mirror descent (\dmd), motivates an autonomous update $\Phi$ to the optimization variable to improve performance for online learning.   
\dmd-\mpc \citep{Boots-RSS-19} incorporates this into an \mpc scheme by incorporating an explicit time shift operator $\Phi$,
so
$\vtheta_{t+1} = \Phi(\vtheta)$,
and
$\vtheta =\argmin_{\vtheta\in\Theta} \nabla l_t(\vtheta_t)\tran\vtheta + \sD_\psi(\vtheta||\vtheta_t)$.
Unlike mirror descent, \dmd-\mpc does not consider an adaptive temperature strategy. 
Moreover, it introduces \cem- and \mppi- like updates by explicitly transforming the objective and preforming gradient descent, rather than estimating the posterior. 

\textbf{Path integral control.}
Path integral theory connects Monte Carlo esimations to partial differential equation (PDE) solutions \citep{kappen2016adaptive}.
Using this, path integral control is used to motivate sample-approximation to the continuous-time Hamiltonian-Jacobian PDE.
Crucially, the exponential transform is introduced to the value function term to make the PDE linear, a requirement to path integral theory. 
Path integral theory is limited to estimating a optimal action trajectory, given an initial state, and requires several additional assumptions on the dynamics and disturbance noise \citep{theodorou2010generalized}. 
Later, a divergence-minimization perspective was applied to path integral methods, to produce algorithms for more general settings and discrete time, referred to both as (information theoretic) \textsc{it-mpc} \citep{williams2018information} and also \mppi \citep{williams2017information}.
This alternative view is closer to \ppi and no longer contains the key assumptions that apply to path integral theory.

\textbf{Variational Inference MPC}
Variational \mpc \citep{okada2020variational} is the same posterior policy iteration scheme outlined in the main text. 
The key difference is a posterior entropy bonus, like \more, added to the objective
\begin{align*}
	\min_q \E_q[R] + \frac{1}{\lambda_1}\KL[q\mid\mid p] + \frac{1}{\lambda_2}\mathbb{H}[q].
\end{align*}
This results, once reparameterized, in the usual posterior update with an annealing coefficient $\kappa$ on the prior
\begin{align*}
	w_n \propto \exp(\alpha r_n)\,p(\va_n)^{-\kappa}.
\end{align*}
Due to hyperparameter sensitivity, this entropy regularization required a normalization step itself
\begin{align}
	p(\va)^{-\kappa} = \exp\left(-\kappa\log p(\va)\right)
	\rightarrow
	\exp\left(-\bar{\kappa}\frac{\log p - \max\log p}{\min\log p - \max\log p}\right).
\end{align}
As a result, the entropy regularization is dynamic in practice.
The mechanism of this update is to increase the weight of low-likelihood samples such that the entropy of the posterior increases.  
In the ablation study, the effect of this regularization was relatively small with limited statistical significance. 
The work also uses a Gaussian mixture model as the variational family for multi-modal action sequences. 
As an \mppi baseline, the authors use \pitwo for \mpc with an adaptive temperature through normalization with $\bar{\alpha}$ set to $10$.

\textbf{Variational Inference \mpc using Tsallis Divergence}
Adopting the `generalized' variational inference approach, Wang et al. replace the \kl divergence with the Tsallis divergence \citep{wang2021variational}, resulting in discontinuous, rather than exponential, expression for the posterior weights
\begin{align}
	w_n &\propto \max(0, 1 + (\gamma-1)\,\alpha \,r_n)^{\frac{1}{\gamma-1}},
	\quad
	\gamma > 0.
\intertext{In practice, like in the \cem, $k$ elites are used to define the $r_*$ threshold}
	w_n &\propto
	\begin{cases}
	\exp\left(\frac{1}{\gamma-1}\log\left(1-\frac{r_n}{r_*}\right)\right) &\text{ if } r_n < r_*,\\
	0, &\text{ otherwise}.
	\end{cases}
\end{align}
The resulting parameterization, essentially combines the \cem- and importance sampling approaches, using elites but weighing them through $\gamma$. 
For the implementation, the number of elites and $\gamma$ were tuned but kept constant during \mpc.
From the perspective outlined in this work, this approach limits the effective sample size at a maximum value, but allows it to drop if there is sufficient range in the elite returns. 
As a result, its optimization is greedier than \cem, resulting in better downstream control performance.
In the paper, the authors compare to \mppi with a static temperature and do not consider adaptive schemes.
In theory, the Tsallis divergence could be combined with \lbps to optimize $\gamma$ adaptively.

\textbf{Stein Variational Inference Model Predictive Control}
Stein variational gradient descent (\textsc{svgd}) is an approximate inference method for learning multi-modal posteriors parameterized by particles and a kernel function,
$q(\cdot) = \sum_n k(\cdot,\vx_n)$.
The key quality of \textsc{svgd} is the kernel, which provides a repulsion force during learning to encourage particle diversity.

For control, \textsc{svgd} can be used to infer multi-modal action sequence which exhibit high entropy for exploration, as done in \textsc{sv-mpc} \citep{lambert2020stein}.  
Due to the limitation of kernel methods with high-dimensional inputs, the kernel is designed to be factorized in time and action.
The kernel design is also Markovian, so it considers the temporally adjacent variables as well, encouraging smoothness.
This design decomposes the kernel into the sum of $H(H{-}1)d_a$ one dimensional kernels for $H$ timesteps. 
Moreover, extrapolation into the future is realized by copying the last timestep (i.e. a zero-order hold). 
\textsc{svgd} requires gradients of the loglikelihood, using backpropagation through the dynamics or Monte Carlo estimates. 
\textsc{sv-mpc} has an additional learning rate hyperparameter and uses an independent Gaussian action prior across timestemps.

\newpage 

\section{Implementation Details}

\begin{algorithm}[bt]
	\caption{Open-loop Episodic Monte Carlo Posterior Policy Iteration}
	\label{alg:MCPIP}
	{\textbf{Requires}:\\
	Markov decision process \texttt{MDP},
	initial policy $\pi_1$,
	posterior strategy \texttt{GibbsPosterior},
	initial state $\vs_0$.
	}
	\begin{algorithmic}[1]
		\For{$i \gets 1$ to $N$}
		\State Sample action sequences and associated parameters $\vtheta$, $\mA^{(n)},\,\vtheta^{(n)} \sim \pi_i(\cdot|\vs_0)$
		\State Obtain returns $R^{(n)} \sim \text{\texttt{MDP}}(\mA^{(n)})$
		\State Compute importance weights, $\vw{\,\leftarrow\,} \texttt{GibbsPosterior}(\mR)$ to estimate 
		$q_\alpha(\mA|\gO)$
		\State Update policy $\pi_{i+1} \leftarrow \texttt{MProjection}(\vw, \vtheta)$, 
		performing
		$\min_\pi \KL[q_\alpha\mid\mid\pi]$
		\EndFor
	\end{algorithmic}
	\label{alg:ppi}
\end{algorithm}

Algorithm \ref{alg:ppi} describes the general routine of Monte Carlo \ppi.
The specific \texttt{GibbsPosterior} update must be chosen, e.g. \reps, \pitwo, \mppi, \lbps or \essps.
\cem corresponds to uniform importance weights applied to the elite samples in \texttt{GibbsPosterior}.
By definition, \mppi has a fixed covariance.
For \mpc, \cem methods reset their covariance each timestep.  
For actuator limits, we treat them as properties of the dynamics, and later fit the posterior using the \emph{clipped} actions.
This recognizes that the applied policy is not necessarily the same as the sampled one and is observed along with the reward.
For feature regression weights, clipping is applied but cannot be incorporated into the fitting, as the model weights are being fit, which may explain its worse performance for \mpc.    

As done in \reps and \more, the temperature optimization uses off-the-shelf minimizers, such as those found in \texttt{scipy.opt.minimize} \citep{2020SciPy-NMeth}. 
While \reps and \more use gradient-based solvers such as \textsc{lbfgs-b}, for \lbps and \essps, the \texttt{brent} method in \texttt{scipy.opt.minimize\_scalar} was found to be slighly faster, presumably due to its quasi-convex objective (see Section \ref{sec:lower_bound}). 

For the matrix normal distribution, a weighted maximum likelihood fit and sampling is very similar to the normal distribution. 
Sampling a matrix normal $\gM\gN(\mM,\mK,\mSigma)$ is achieved using
$$
\mX^{(n)} = \mM + \mA\mW^{(n)}\mB,
\quad
\mK = \mA\mA\tran,
\quad
\mSigma = \mB\tran\mB,
\quad
W_{ij}^{(n)}\sim\gN(0, 1).
$$
The weighted maximum likelihood fit of the input covariance is computed using \citep{dutilleul1999mle} 
$$\mK = \textstyle\sum_n w_n(\mX^{(n)}-\mM)(\mX^{(n)}-\mM)\tran\mSigma\inv,
\quad
\text{where }
\sum_n w_n = 1.$$

For the feature approximation of the squared exponential kernel, we used \rbf features and quadrature \rff{}s (\qrff).
\rbf features require careful normalization in order to approximate the \se kernel at the limiting case~\citep{gpml}. 
A $d$-dimensional \rbf feature for a \se kernel approximation with lengthscale $l$ is defined as
\begin{align*}
	\bm{\phi}(t) = \frac{1}{\sqrt{\sqrt{\pi}\,d \lambda}}\exp\left(-\frac{(t-\vc)^2}{2\lambda^2}\right),
\end{align*}
where $\lambda = l/\sqrt{2}$.
The $d$-dimensional centers $\vc$ are linearly placed along the task horizon.

Quadrature random Fourier features require $d = (2\nu)^k$ features for order $\nu$ and input dimension $k$.
Therefore, for time-series we require $2\nu$ features where
\begin{align*}
	\phi_j(t) = \begin{cases}
		w_j\cos(\omega_j t) \quad j \leq \nu\\
		w_{j-\nu}\sin(\omega_{j-\nu} t) \quad \nu < j \leq 2\nu
	\end{cases}
\end{align*}
Gauss-Hermite quadrature provides points $\vu$ and weights $\vv$ for a given order, to be used form approximating integrals.
In \qrff{}s, the Monte Carlo integration with Gaussian frequencies \citep{rahimi2007random} is replaced with quadrature.
As a results, $w_i = 2v_j/\sqrt{\pi}$
Incorporating the lengthscale, $\omega_i = \sqrt{2}u_i/l$.
See Mutn\'y et al. for a more in-depth description of \qrff{}s, including for higher-dimensional inputs~\citep{mutny2018efficient}.

A practical issue in stochastic search methods is maintaining a sufficiently exploratory search distribution.
In accordance with Bayesian methods, \ppi methods maintain a belief and do not keep a fixed covariance like \mppi.
However, for some tasks it was found that the search distribution had insufficient variance to solve the task effecively, therefore, we introduced an adjustment to 
\Eqref{eq:cov_update} to `anneal' the covariance, recognizing the the update subtracts a likelihood-based term from the prior.
Using
\begin{align*}
\mSigma_{t_3:t_4,t_3:t_4|\gO} &= \mSigma_{t_3:t_4,t_3:t_4} - \gamma\mSigma_{t_3:t_4,t_1:t_2}
	\mLambda_{t_1:t_2}
	\mSigma_{t_3:t_4,t_1:t_2}\tran, 
	\quad
	0 \leq \gamma \leq 1,
\end{align*}
we perform standard \ppi for $\gamma{\,=\,}1$, but adopt an \mppi-like approach for $\gamma{\,=\,}0$.
For \texttt{HumanoidStandup-v2}, $\gamma{\,=\,}0.5$ was required to transition from standing to stabilization effectively. 
Appendix \ref{sec:mpc_ablation} provides ablation studies of this aspect. 

Investigating the runtime of these methods, comparing \icem, \mppi and \lbps in Table \ref{tab:time}, \ppi methods are slower than \cem/\icem, increasing with sample size. 
Rather than due to the $\alpha$ optimization or kernel-based prior construction, the bottleneck is using all the samples in the matrix normal \textsc{mle} step for the covariance, computed using \texttt{einsum} operations on the sampled parameters. 
Based on this study, one approximation to speed up \ppi methods is use the \ess to reduce the number of samples used in the matrix normal \textsc{mle} step, pruning samples that have negligible weight, like in \cem.
\begin{table}[!tb]
	\centering
		\begin{tabular}{llll}
			Algorithm & \multicolumn{3}{c}{\texttt{n\_samples}}\\\
			& 16 & 128 & 1024
			\\\toprule
			\mppi (white noise) & 0.06 & 0.37 & 2.63 \\
			\icem (coloured noise) & 0.07 & 0.18 & 1.08 \\
			\lbps (\se kernel) & 0.06 & 0.36 & 2.74 \\
			\mppi (\se kernel) & 0.06 & 0.36 & 2.74 \\
			\lbps (\rbf features) & 0.06 & 0.37 & 2.79 \\
			\bottomrule     
		\end{tabular}
		\vspace{0.5em}
		\caption{
			Wall-clock time (s) of one \mpc calculation for \hsu.
			Averaged over 10 timesteps and 5 seeds. 
			Computation was performed on a AMD Ryzen 9 3900X 12-Core Processor @ 3.8GHz, parallelized across 24 processes.
		}
		\label{tab:time}
	\end{table}

\newpage

\section{Extended Experimental Results}
\vspace{-0.5em}
\subsection{Black-box Optimization}
\label{sec:bbo_res}
\begin{figure}[!h]
	\begin{minipage}{\textwidth}
	\centering
	\vspace{-1em}
	\begin{tikzpicture}
	
	\begin{axis}[
		height=2cm,
		hide axis,
		xmin=10, xmax=50,
		ymin=0, ymax=1.0,
		legend cell align={center},
		legend columns=3,
		legend style={/tikz/every even column/.append style={column sep=0.3cm}, draw=none},
		]
		]
		
		\addlegendimage{semithick, darkviolet1910191, mark=-}
		\addlegendentry{$k=3$};
		\addlegendimage{semithick, darkturquoise0191191, mark=-}
		\addlegendentry{$k=8$};
		\addlegendimage{semithick, blue, mark=-}
		\addlegendentry{$k=16$};
	\end{axis}
	
\end{tikzpicture}
	\vspace{-1em}
	\end{minipage}
	\includegraphics{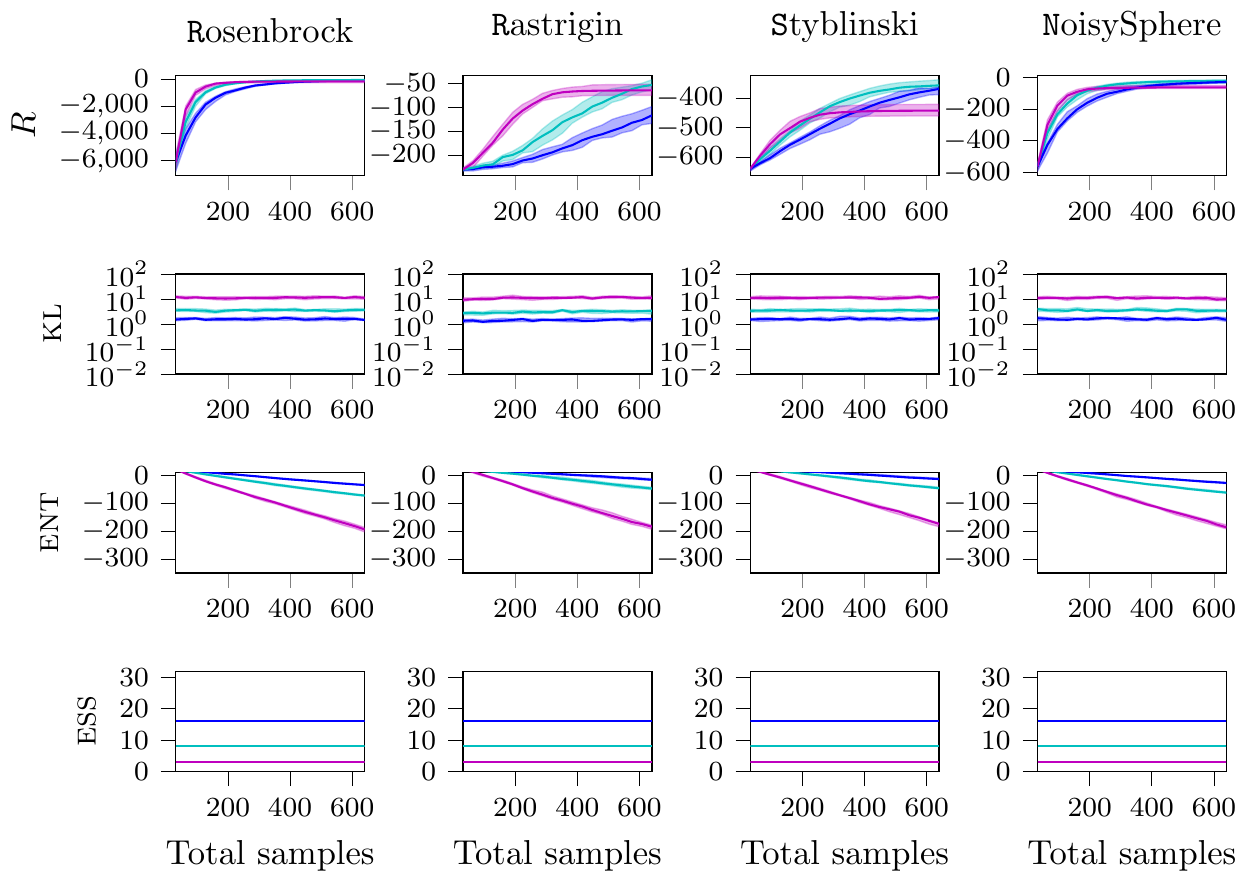}
	\caption{Black-box optimization with \cem and Monte Carlo sampling, with 32 samples over 20 episodes, displaying quartiles over 25 seeds.
	$k$ is the number of `elite' samples.
	}
	\label{fig:bbo1}
	\vspace{-1em}
\end{figure}

\begin{figure}[!h]
	\begin{minipage}{\textwidth}
	\centering
	\vspace{-1em}
	\include{fig/opt/essps_legend}
	\vspace{-1em}
	\end{minipage}
	\includegraphics{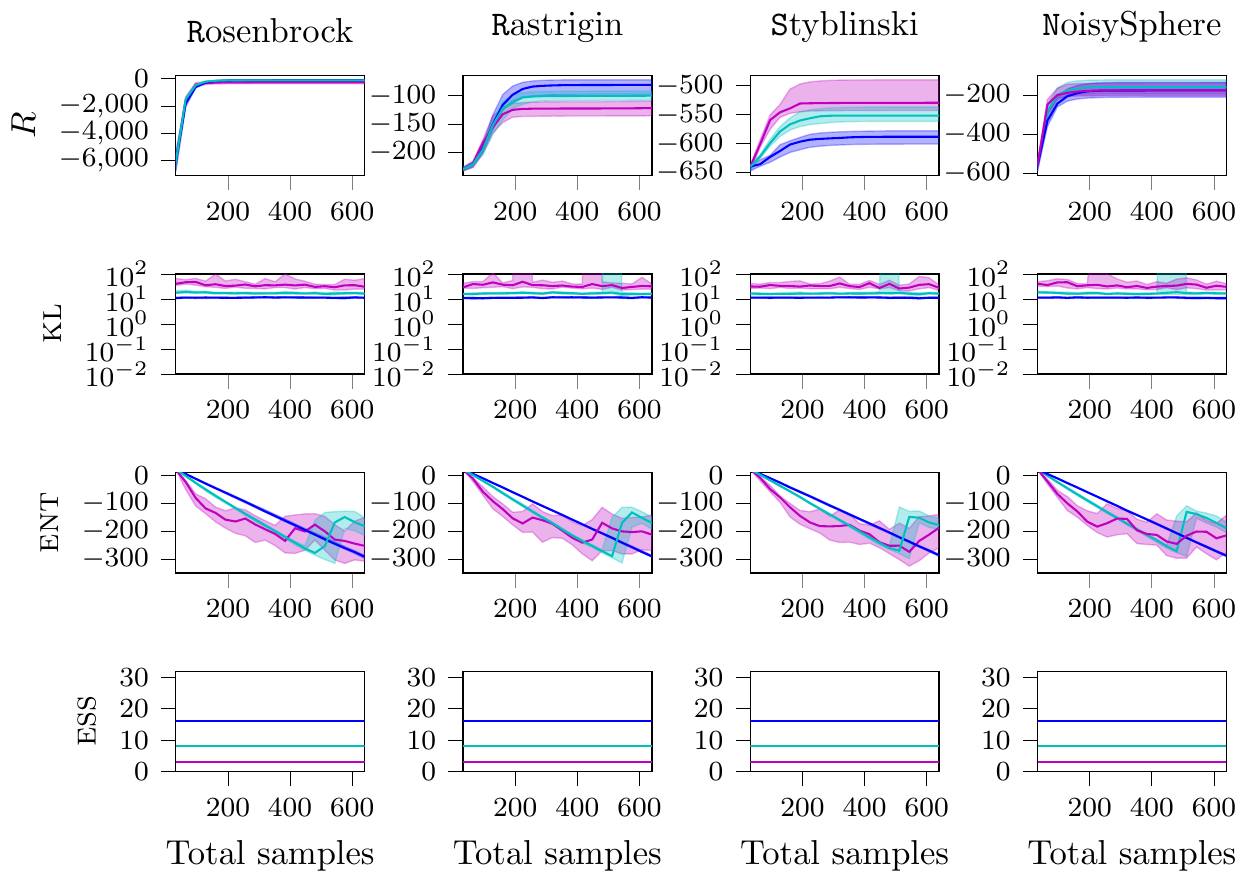}
	\caption{Black-box optimization with \essps and Monte Carlo sampling, with 32 samples over 20 episodes, displaying quartiles over 25 seeds. 
	$N^*$ is the desired effective sample size.
	}
	\label{fig:bbo_essps}
\end{figure}

\newpage

\begin{figure}[!h]
	\begin{minipage}{\textwidth}
	\centering
	\vspace{-1em}
	\begin{tikzpicture}
	
	\begin{axis}[
		height=2cm,
		hide axis,
		xmin=10, xmax=50,
		ymin=0, ymax=1.0,
		legend cell align={center},
		legend columns=3,
		legend style={/tikz/every even column/.append style={column sep=0.3cm}, draw=none},
		]
		]
		
		\addlegendimage{semithick, blue, mark=-}
		\addlegendentry{$\epsilon=0.1$};
		\addlegendimage{semithick, darkviolet1910191, mark=-}
		\addlegendentry{$\epsilon=1$};
		\addlegendimage{semithick, darkturquoise0191191, mark=-}
		\addlegendentry{$\epsilon=10$};
	\end{axis}
	
\end{tikzpicture}
	\vspace{-1em}
	\end{minipage}
	\includegraphics{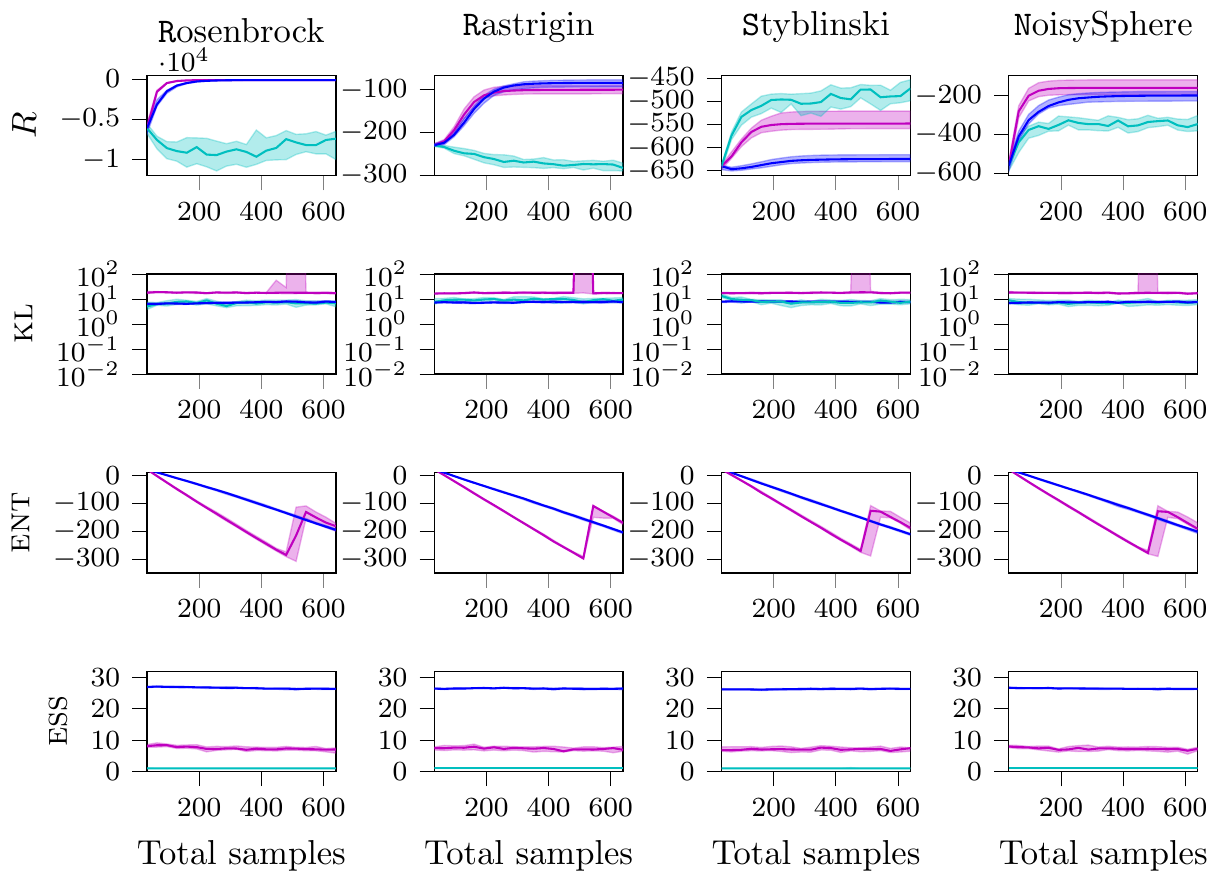}
	\caption{Black-box optimization with \reps and Monte Carlo sampling, with 32 samples over 20 episodes, displaying quartiles over 25 seeds.
	$\epsilon$ is the \kl bound.
	}
\end{figure}

\begin{figure}[!h]
	\begin{minipage}{\textwidth}
	\centering
	\vspace{-1em}
	\begin{tikzpicture}
	
	\begin{axis}[
		height=2cm,
		hide axis,
		xmin=10, xmax=50,
		ymin=0, ymax=1.0,
		legend cell align={center},
		legend columns=3,
		legend style={/tikz/every even column/.append style={column sep=0.3cm}, draw=none},
		]
		]
		
		\addlegendimage{semithick, blue, mark=-}
		\addlegendentry{$\delta=0.1$};
		\addlegendimage{semithick, darkviolet1910191, mark=-}
		\addlegendentry{$\delta=0.5$};
		\addlegendimage{semithick, darkturquoise0191191, mark=-}
		\addlegendentry{$\delta=0.9$};
	\end{axis}
	
\end{tikzpicture}
	\vspace{-1em}
	\end{minipage}
	\includegraphics{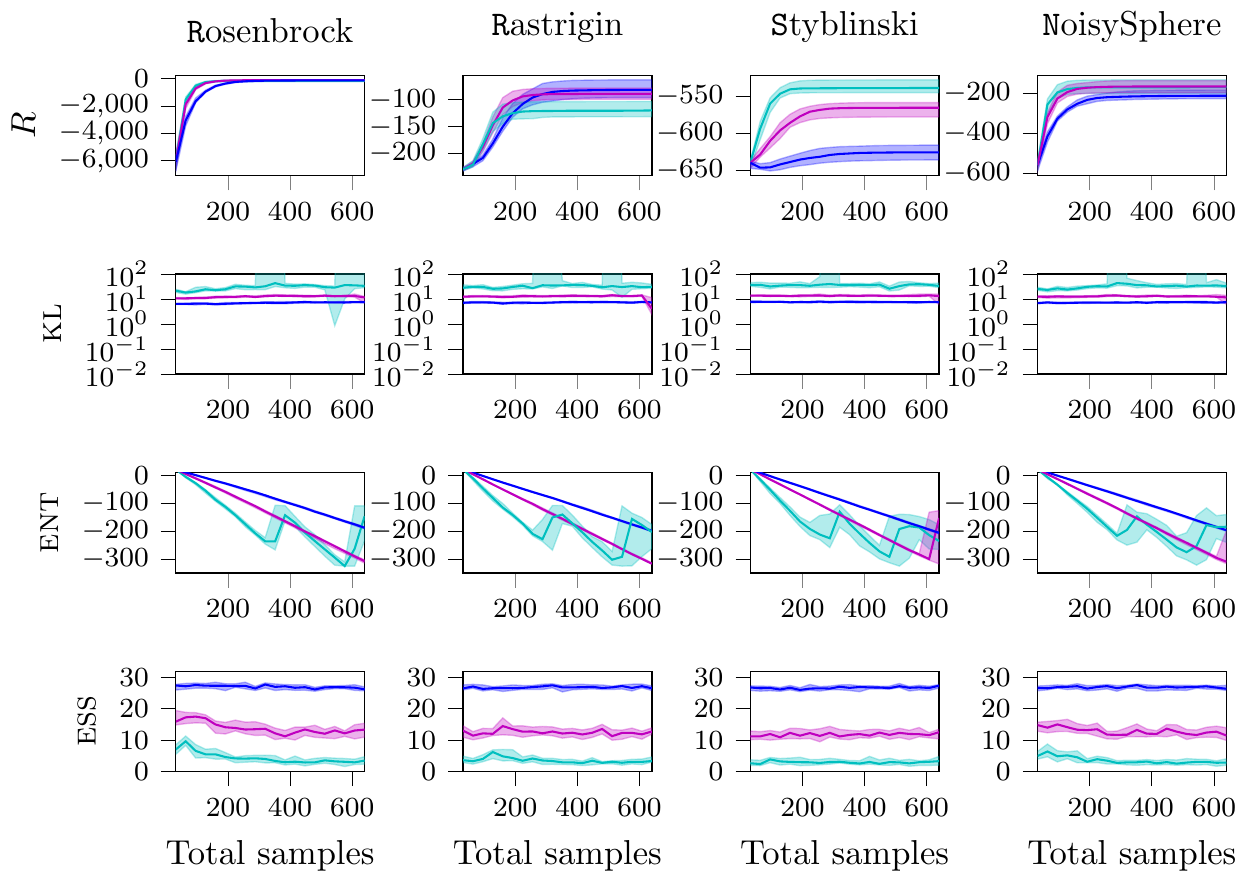}
	\caption{Black-box optimization with \lbps and Monte Carlo sampling, with 32 samples over 20 episodes, displaying quartiles over 25 seeds.
	$\delta$ is the probability of the lower bound.
	}
	\label{fig:bbo4}
\end{figure}

\newpage

\subsection{Policy Search}
\label{sec:ps_res}

\begin{figure}[!h]
	\begin{minipage}{\textwidth}
		\centering
		\begin{tikzpicture}
	
	\begin{axis}[
		height=2cm,
		hide axis,
		xmin=10, xmax=50,
		ymin=0, ymax=1.0,
		legend cell align={center},
		legend columns=5,
		legend style={/tikz/every even column/.append style={column sep=0.3cm}, draw=none},
		]
		]
		
		\addlegendimage{semithick, red, mark=*, mark size = 1}
		\addlegendentry{\reps};
		\addlegendimage{semithick, darkturquoise0191191, mark=*, mark size = 1}
		\addlegendentry{\lbps};
		\addlegendimage{semithick, blue, mark=*, mark size = 1}
		\addlegendentry{\essps};
	\end{axis}
	
\end{tikzpicture}
	\end{minipage}
	\includegraphics{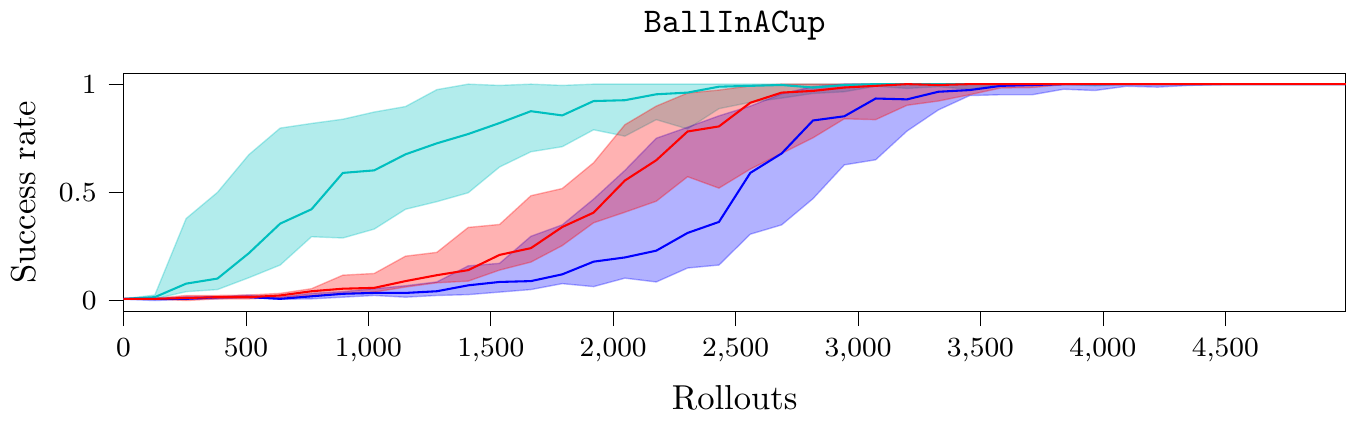}
	\caption{Success rate for policy search with \rbf features over 20 seeds, using 128 rollout samples per episode. Displaying uncertainty in quartiles}
	\label{fig:ps_rbf}
\end{figure}

\begin{figure}[!h]
	\begin{minipage}{\textwidth}
		\centering
		\begin{tikzpicture}
	
	\begin{axis}[
		height=2cm,
		hide axis,
		xmin=10, xmax=50,
		ymin=0, ymax=1.0,
		legend cell align={center},
		legend columns=5,
		legend style={/tikz/every even column/.append style={column sep=0.3cm}, draw=none},
		]
		]
		
		\addlegendimage{semithick, red, mark=*, mark size = 1}
		\addlegendentry{\reps};
		\addlegendimage{semithick, darkturquoise0191191, mark=*, mark size = 1}
		\addlegendentry{\lbps};
		\addlegendimage{semithick, blue, mark=*, mark size = 1}
		\addlegendentry{\essps};
	\end{axis}
	
\end{tikzpicture}
	\end{minipage}
	\includegraphics{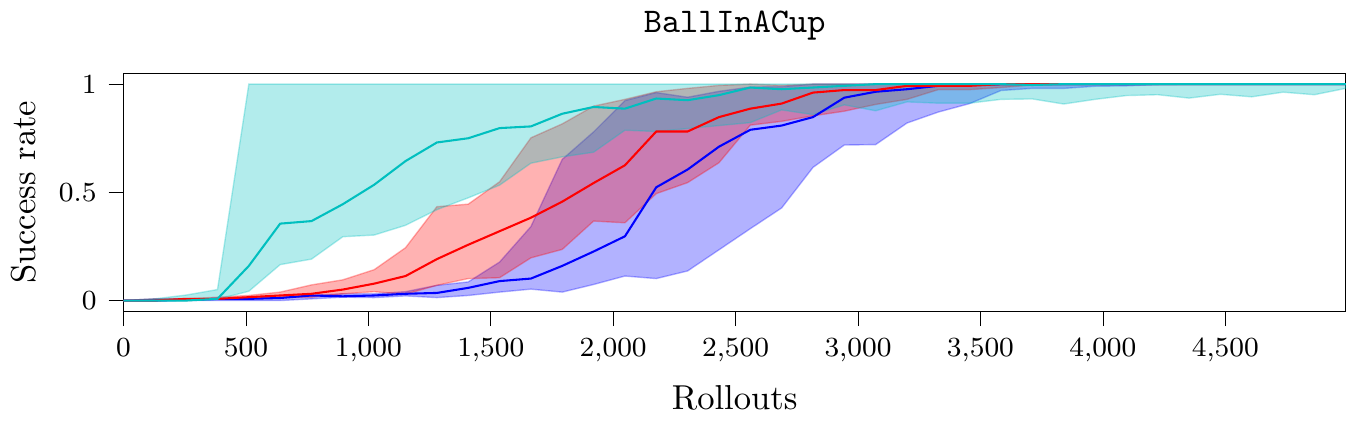}
	\caption{Success rate for policy search with \rff features over 20 seeds, using 128 rollout samples per episode. Displaying uncertainty in quartiles}
	\label{fig:ps_rff}
\end{figure}

\newpage 

\subsection{Model Predictive Control}

\begin{figure}[!h]
	\begin{minipage}{\textwidth}
	\centering
	\begin{tikzpicture}
	
	\begin{axis}[
		hide axis,
		height=2cm,
		xmin=10, xmax=50,
		ymin=0, ymax=1.0,
		legend cell align={center},
		legend columns=5,
		legend style={/tikz/every even column/.append style={column sep=0.3cm}, draw=none},
		]
		]
		
		\addlegendimage{semithick, red, mark=*, mark size = 1}
		\addlegendentry{\cem};
		\addlegendimage{semithick, darkturquoise0191191, mark=*, mark size = 1}
		\addlegendentry{\lbps};
		\addlegendimage{semithick, blue, mark=*, mark size = 1}
		\addlegendentry{\essps};
		\addlegendimage{semithick, goldenrod1911910, mark=*, mark size = 1}
		\addlegendentry{\mppi};
		\addlegendimage{semithick, green01270, mark=*, mark size = 1}
		\addlegendentry{\pitwo};
	\end{axis}
	
\end{tikzpicture}
	\vspace{-1em}
	\end{minipage}
	\includegraphics{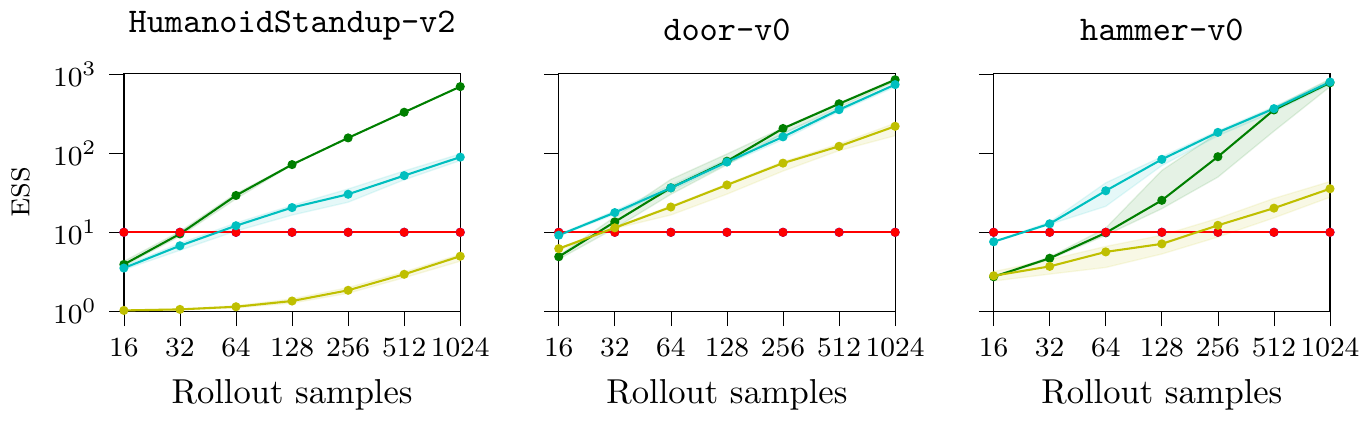}
	\caption{Average \ess for Monte Carlo \mpc with white noise priors. Displaying quartiles over 50 seeds.}
	\label{fig:ess_wn}
\end{figure}

\begin{figure}[!h]
	\begin{minipage}{\textwidth}
	\centering
	\begin{tikzpicture}
	
	\begin{axis}[
		height=2cm,
		hide axis,
		xmin=10, xmax=50,
		ymin=0, ymax=1.0,
		legend cell align={center},
		legend columns=3,
		legend style={/tikz/every even column/.append style={column sep=0.3cm}, draw=none},
		]
		]
		
		\addlegendimage{semithick, darkviolet1910191, mark=*, mark size = 1}
		\addlegendentry{\icem (coloured noise)};
		\addlegendimage{semithick, darkturquoise0191191, mark=*, mark size = 1}
		\addlegendentry{\lbps (\se kernel)};
		\addlegendimage{semithick, blue, mark=*, mark size = 1}
		\addlegendentry{\essps (\se kernel)};
		\addlegendimage{semithick, goldenrod1911910, mark=*, dashed, mark size = 1}
		\addlegendentry{\mppi (smooth actions)};
		\addlegendimage{semithick, goldenrod1911910, mark=*, dash pattern=on 1pt off 3pt on 3pt off 3pt, mark size = 1}
		\addlegendentry{\mppi (smooth noise)};
		\addlegendimage{semithick, goldenrod1911910, mark=*, mark size = 1}
		\addlegendentry{\mppi (\se kernel)};
	\end{axis}
	
\end{tikzpicture}
	\vspace{-1em}
	\end{minipage}
	\includegraphics{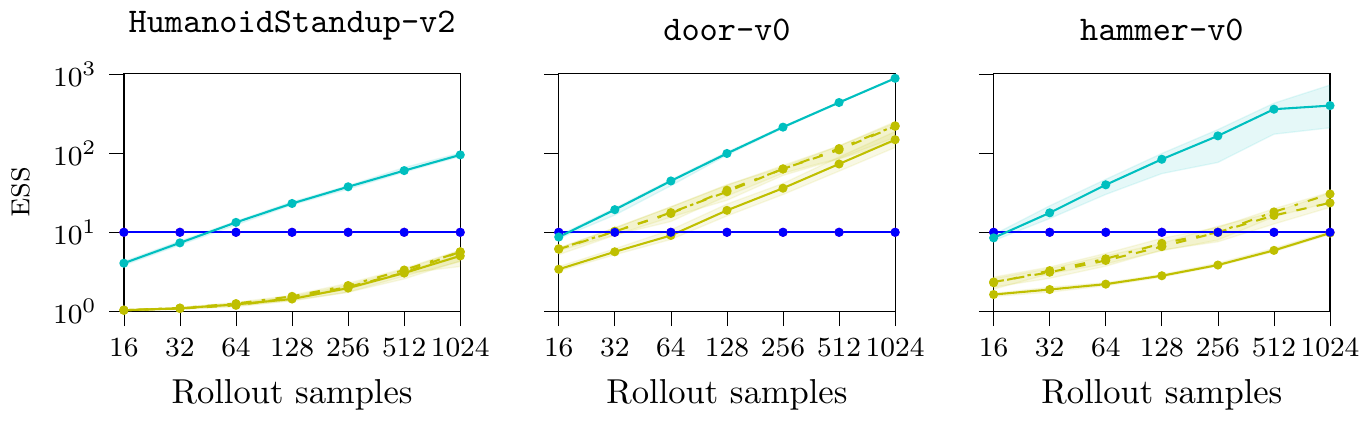}
	\caption{Average \ess for Monte Carlo \mpc with smooth noise priors. Displaying quartiles over 50 seeds.
	i\cem and \essps both have an \ess of 10, due to \icem having 10 elites.}
	\label{fig:ess_smooth}
\end{figure}

\newpage
\subsubsection{\mpc with finite feature approximations}

\begin{figure}[!h]
	\begin{minipage}{\textwidth}
		\centering
		\begin{tikzpicture}
	
	\begin{axis}[
		height=2cm,
		hide axis,
		xmin=10, xmax=50,
		ymin=0, ymax=1.0,
		legend cell align={center},
		legend columns=5,
		legend style={/tikz/every even column/.append style={column sep=0.3cm}, draw=none},
		]
		]
		\addlegendimage{semithick, darkturquoise0191191, mark=*, mark size = 1}
		\addlegendentry{\lbps};
		\addlegendimage{semithick, blue, mark=*, mark size = 1}
		\addlegendentry{\essps};
		\addlegendimage{semithick, black, dashed}
		\addlegendentry{\rff};
		\addlegendimage{semithick, black, dotted}
		\addlegendentry{\rbf};

	\end{axis}
	
\end{tikzpicture}
		\vspace{-1em}
	\end{minipage}
	\includegraphics{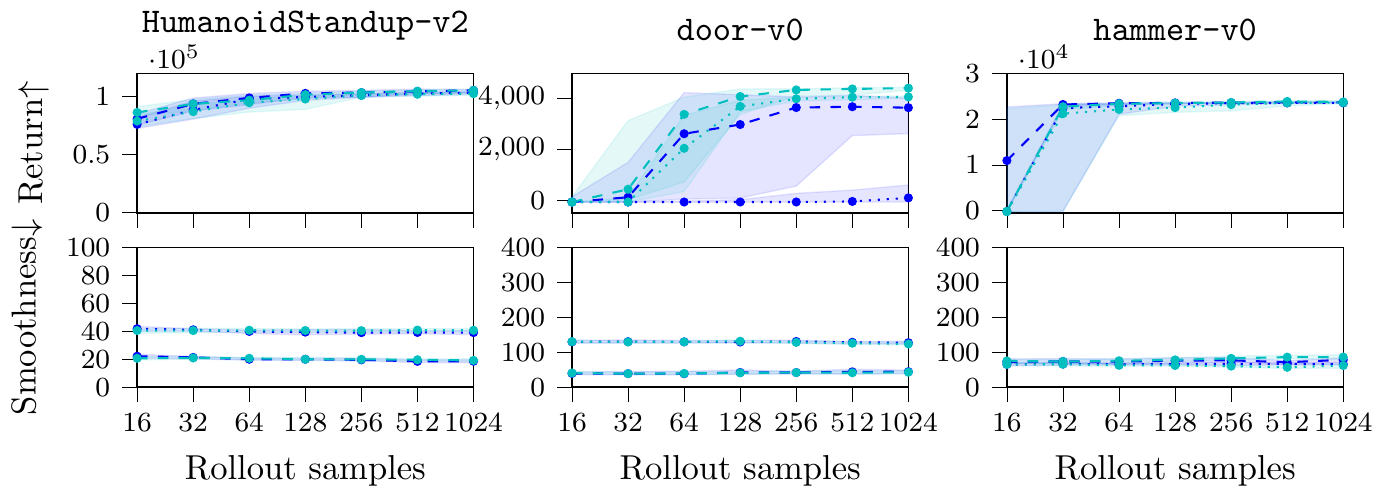}
	\caption{Monte Carlo \mpc with with \rbf and \rff policies. Displaying quartiles over 50 seeds.}
	\label{fig:mpc_ret_feat}
\end{figure}

\begin{figure}[!h]
	\includegraphics{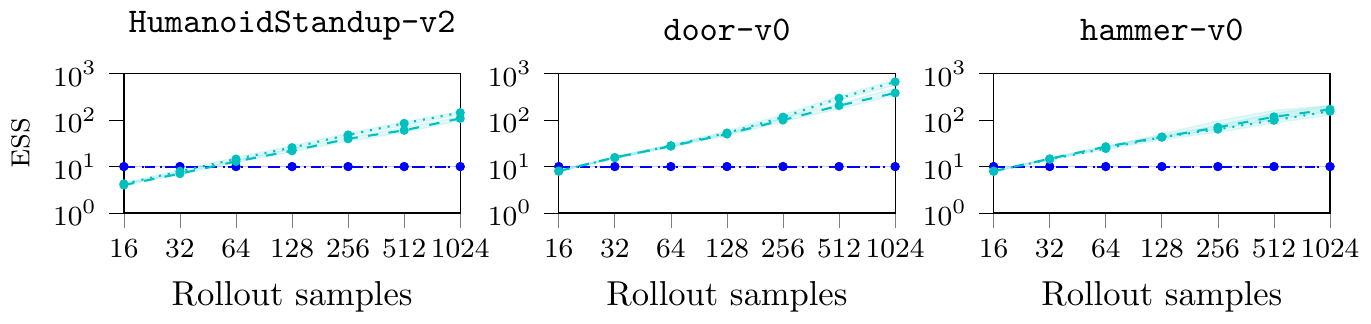}
	\caption{Average \ess for Monte Carlo \mpc with \rbf and \rff policies. Displaying quartiles over 50 seeds.}
	\label{fig:ess_feat}
\end{figure}

\begin{figure}[!h]
	\includegraphics{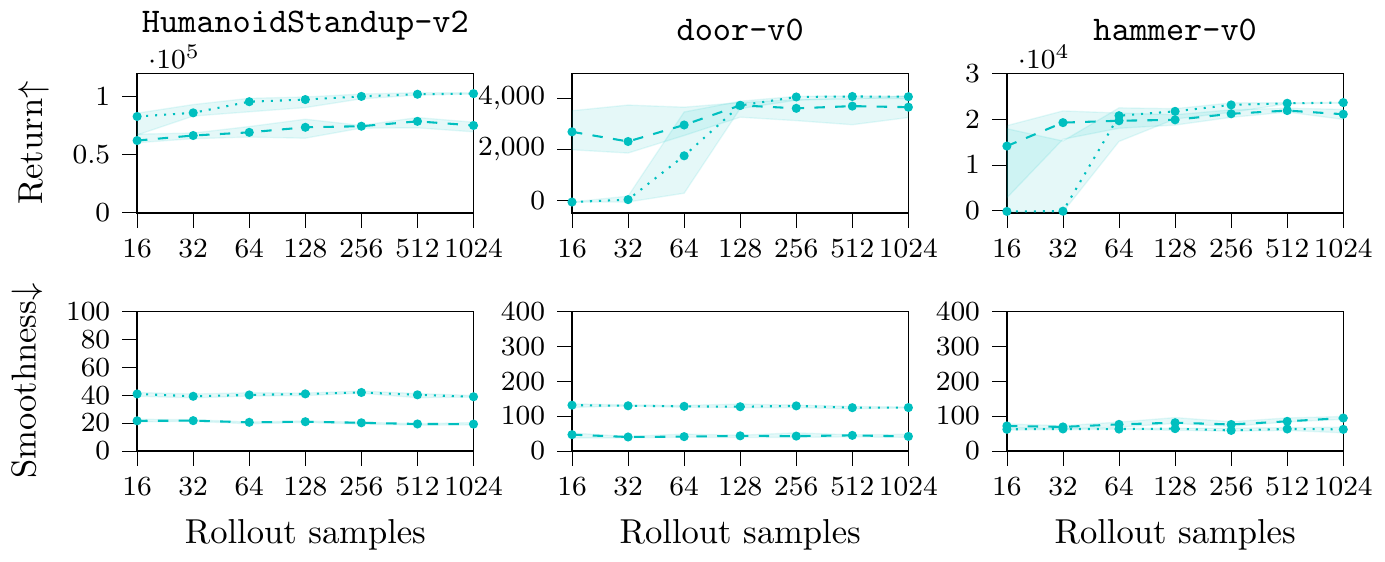}
	\caption{Monte Carlo \mpc with \rbf and \rff policies, using a (shrinking) planning horizon of $250$ rather than $30$, the full task duration. Displaying quartiles over 20 seeds.}
	\label{fig:ret_feat_long}
\end{figure}

\newpage

\subsubsection{Visualizing actuation profiles}
\label{sec:profiles}

To accompany the smoothness metric used in the \mpc evaluation,
we showed how the \se kernel with \ppi produces smoother and lower amplitude policies than alternative priors.
Due to the high-dimensional action spaces, we depict all actions overlapped as they are normalized.  

\begin{figure}[!h]
	\includegraphics{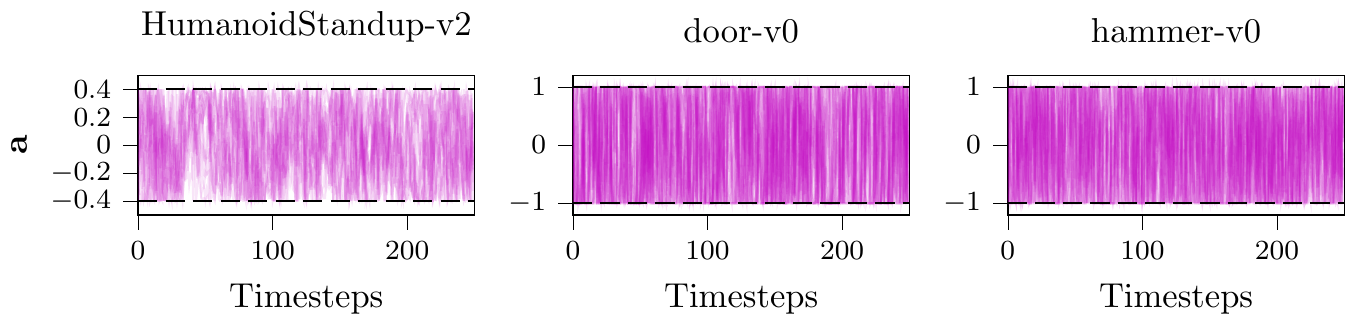}
	\caption{\icem \mpc, with coloured noise, action sequence using 16 rollouts.}
\end{figure}

\begin{figure}[!h]
	\includegraphics{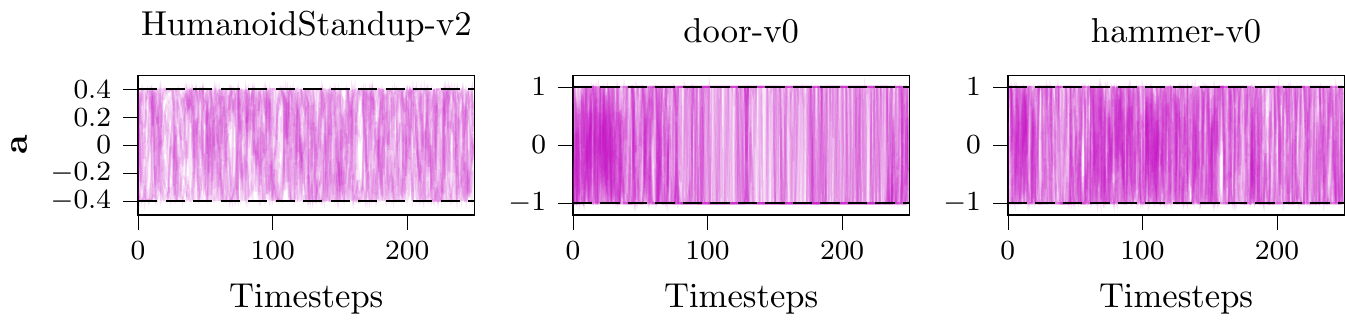}
	\caption{\icem \mpc, with coloured noise, action sequence for using 1024 rollouts.}
\end{figure}

\begin{figure}[!h]
	\includegraphics{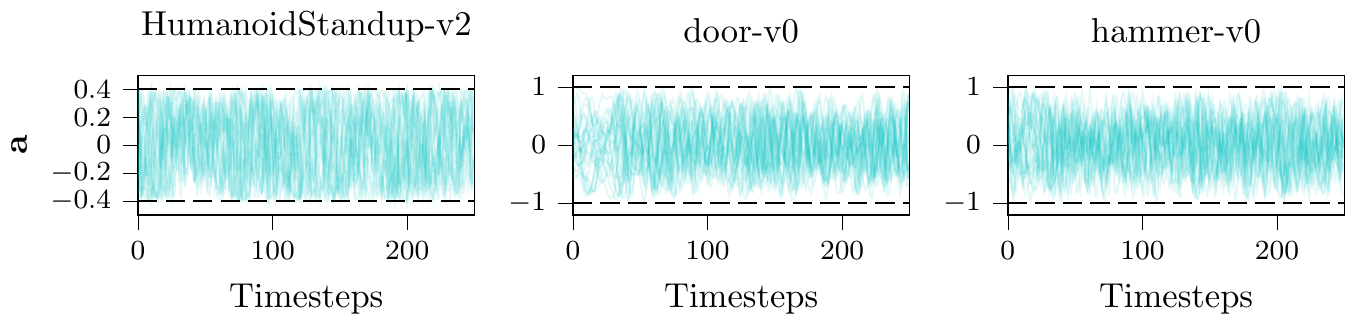}
	\caption{\lbps \mpc, using the \se kernel, action sequence using 16 rollouts.}
\end{figure}

\begin{figure}[!h]
	\includegraphics{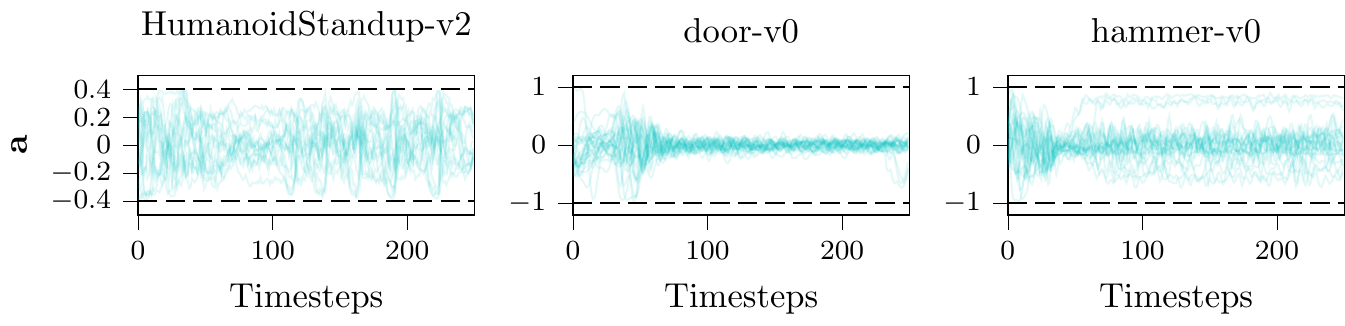}
	\caption{\lbps \mpc, using the \se kernel, action sequence using 1024 rollouts.}
\end{figure}

\begin{figure}[!h]
	\includegraphics{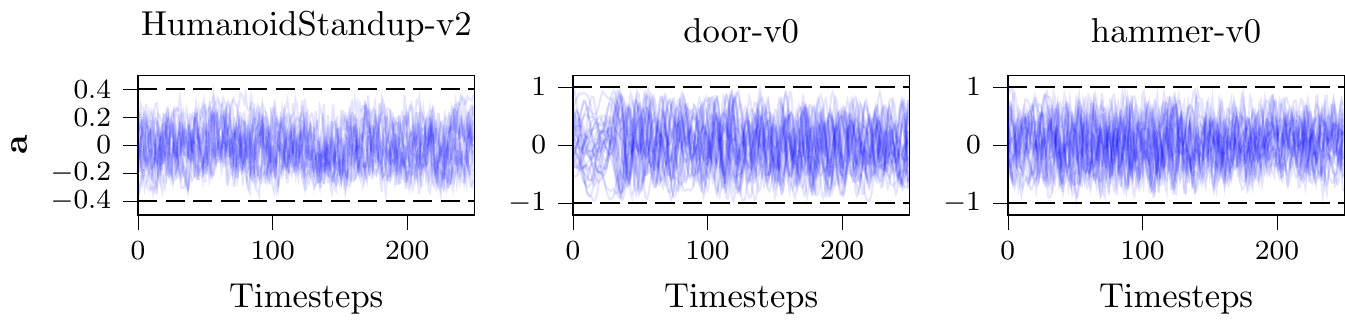}
	\caption{\essps \mpc, using the \se kernel, action sequence using 16 rollouts.}
\end{figure}

\begin{figure}[!h]
	\includegraphics{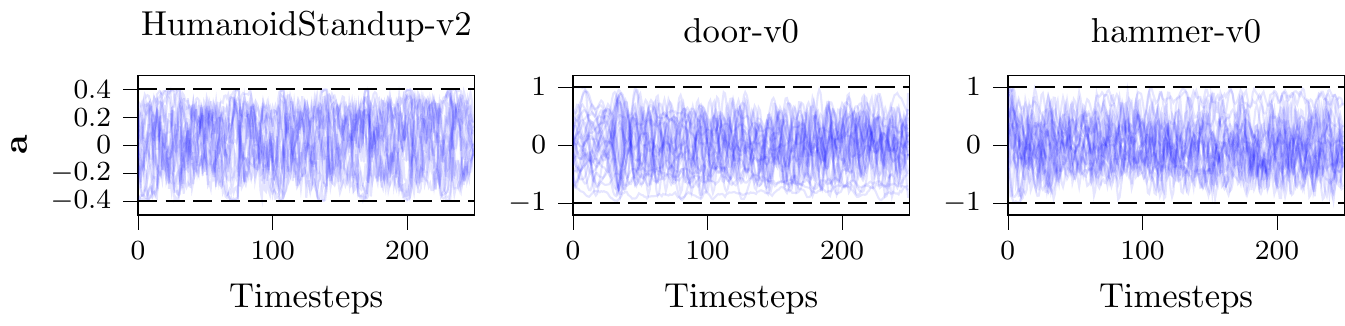}
	\caption{\essps \mpc, using the \se kernel, action sequence using 1024 rollouts.}
\end{figure}

\begin{figure}[!h]
	\includegraphics{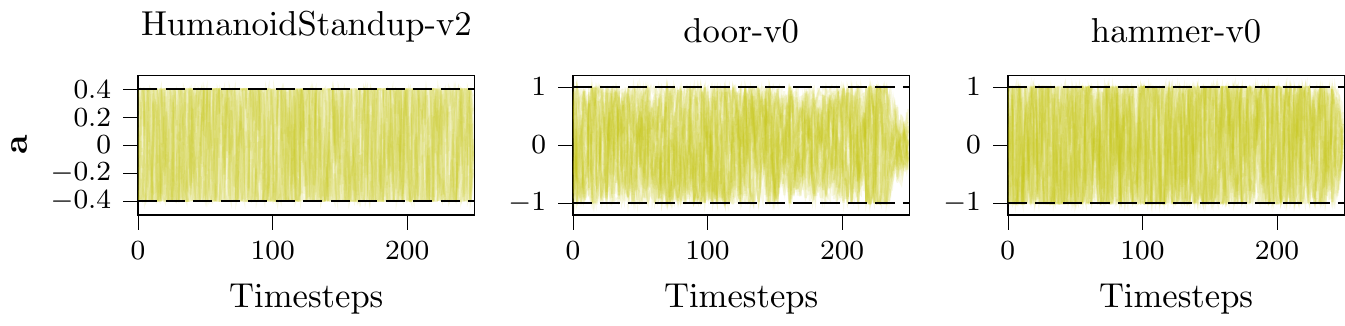}
	\caption{\mppi \mpc, using the \se kernel, action sequence using 16 rollouts.}
\end{figure}

\begin{figure}[!h]
	\includegraphics{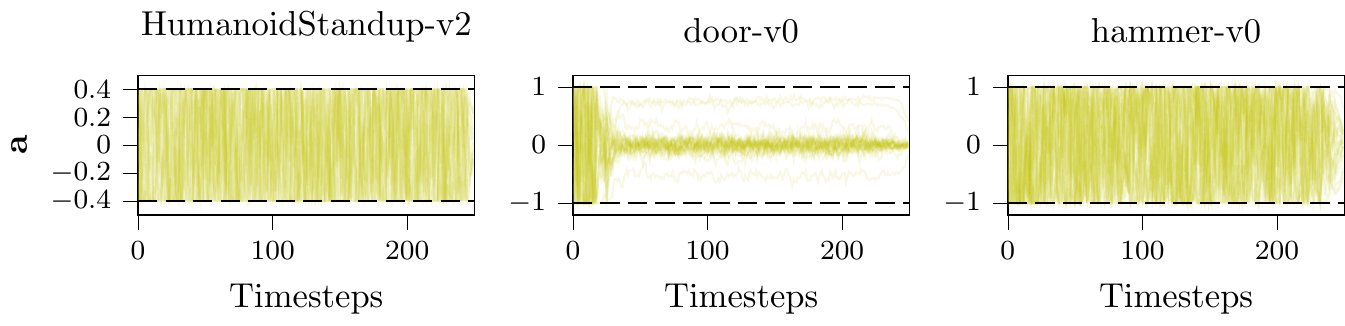}
	\caption{\mppi \mpc, using the \se kernel, action sequence using 1024 rollouts.}
\end{figure}

\clearpage\newpage

\begin{figure}[!h]
	\includegraphics{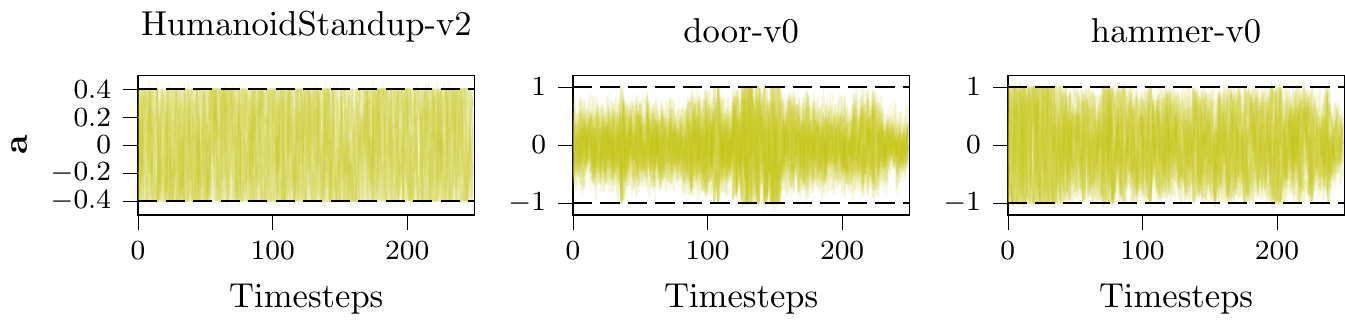}
	\caption{\mppi \mpc, using smooth action noise, action sequence using 16 rollouts.}
\end{figure}

\begin{figure}[!h]
	\includegraphics{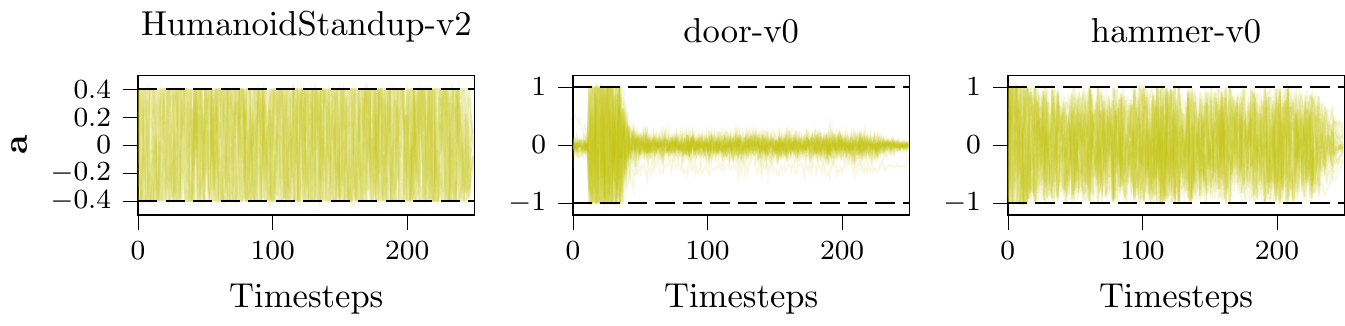}
	\caption{\mppi \mpc, using smooth action noise, action sequence using 1024 rollouts.}
\end{figure}

\begin{figure}[!h]
	\includegraphics{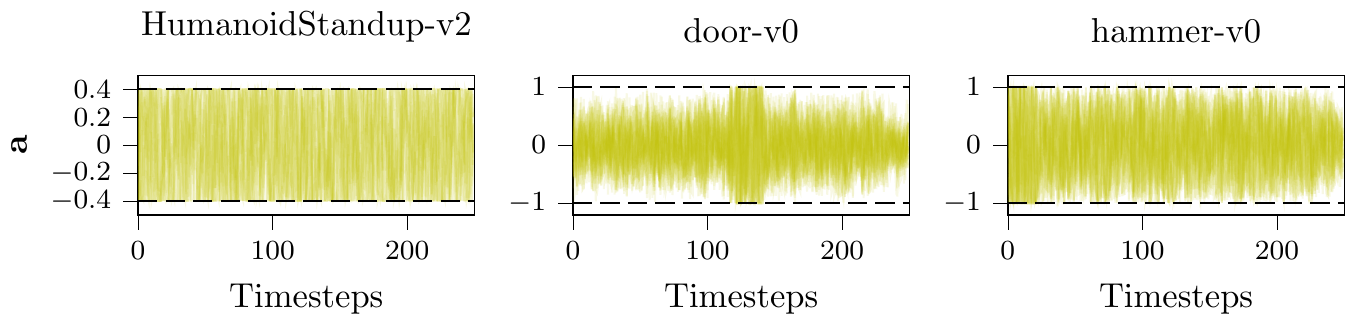}
	\caption{\mppi \mpc, using smooth exploration noise, action sequence using 16 rollouts.}
\end{figure}

\begin{figure}[!h]
	\includegraphics{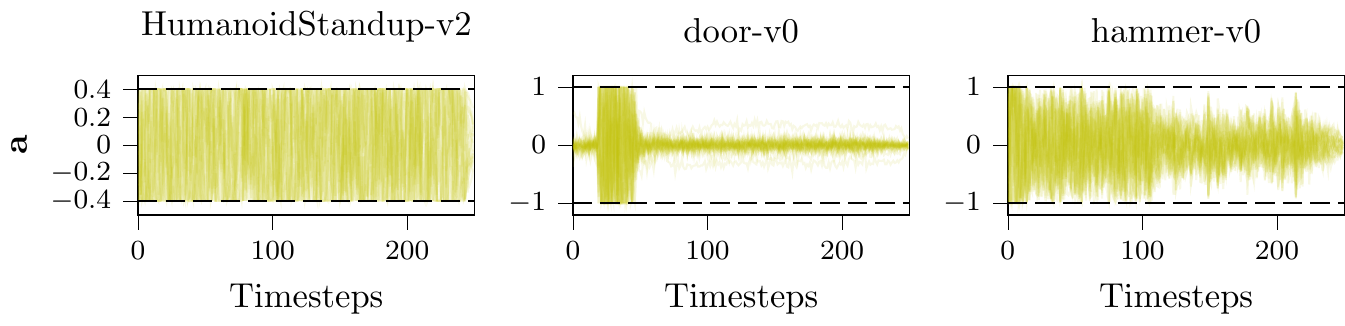}
	\caption{\mppi \mpc, using smooth exploration noise, action sequence using 1024 rollouts.}
\end{figure}

\clearpage\newpage 

\subsubsection{Learning priors from data}
\label{sec:priors_from_data}

A benefit of the Gaussian process view of policy priors is the ability to perform model selection from data, rather the hyperparameter tuning.
This data could be expert demonstrations or a partial solution, such as the `warm start' in \mpc. 

We investigate data driven priors by taking expert demonstrations from human and \rl agents, analyzing them using the matrix normal distribution.
The covariances of the matrix normal allow us to assess the stationarity of the temporal correlations, as well as the correlations between actions.
For example, in Figures \ref{fig:mavn_door}--\ref{fig:mavn_hammer} we can see non-stationarity of \texttt{door-v0} and \texttt{hammer-v0}, due to the motions before and after completing the task, whereas \texttt{HumanoidStandUp-v2} appears surprisingly stationary due to the task having `stand up' and `stabilize' components. 
However, we found the smoothness that benefited \mpc agents in Section \ref{sec:mpc} was not present in the demonstrations. 
For \texttt{door-v0} and \texttt{hammer-v0}, the action space is desired joint positions and the demonstrations were collected using motion capture technology \citep{Rajeswaran-RSS-18}.
With this in mind, the non-smoothness may be the result of motion capture artefacts or the inverse kinematics used.
While the video results of the demonstrations do not suggest rough motion, the action sequences in the dataset do appear rough (Figures \ref{fig:mavn_door}--\ref{fig:mavn_hammer}). 
Moroever, there may be unknown components of the control stack that smooth out the desired joint setpoints. 
However, the issue may lie in the matrix normal factorization, which assumes each action share the same temporal correlations. 
This coupling may result in missing smooth sequences if many dimensions have a rough actions. 

For \texttt{HumanoidStandUp-v2}, we train and use demonstrations from a \texttt{gac} (generative actor critic) agent \citep{peng2021generative}.
To our knowledge, this is the only model-free deep \rl algorithm that solves \texttt{HumanoidStandUp-v2}.
However, the policy learned by \textsc{gac} is a bang-bang controller that operates at the action limits. 
As a result, the smoothness measure introduced in Section \ref{sec:mpc} `breaks', as the norm of this action sequence is \emph{constant}, suggesting a smoothness of 0.  
Looking per action independently, the smoothness metric varies around $100$ to $300$. 
This result suggests that the IID noise used in deep \rl exploration may influence its optimal policy estimate towards rough behaviors such as bang-bang control, which limits their usefulness as expert demonstrators.  

The second quantity of interest is the action covariance $\mSigma$. 
The demonstrations of \texttt{door-v0} and \texttt{hammer-v0} suggest that the independence assumption of $\mSigma$ is a reasonable one. 
The action covariance of \texttt{HumanoidStandUp-v2} depicts much stronger correlations between actions, which could be used to improve exploration through coordination.

\begin{table}[!tb]
	\centering
	\label{tab:data_driven}
		\begin{tabular}{llll}
			Setting & Return & Smoothness & Lengthscale
			\\\toprule
			\texttt{door-v0} (expert) & [3301, 4001, 4028]
			 & [139 145 150] & $\num{1e-5}$\\
			\texttt{door-v0} (\lbps  \mpc) & [4326, 4370, 4388] & [50, 55, 58] & 0.05\\
			\texttt{hammer-v0} (expert) & [14042, 17054, 21272]
			& [142, 147, 151] & $\num{1e-5}$\\
			\texttt{hammer-v0} (\lbps  \mpc) & [23810, 23828, 23904] & [51, 53, 62] & 0.025\\
			\texttt{HumanoidStandup-v2} (\textsc{gac}) & [92429, 92972, 93300] & --$^\dagger$ & $\num{1e-5}$ \\
			\texttt{HumanoidStandup-v2} (\lbps \mpc) & [99902 100147 101337] & [20 22 23] & 0.05 \\
			\bottomrule     
		\end{tabular}
	\vspace{0.5em}
	\caption{
	Comparing expert demonstrations to \ppi \mpc performance w.r.t. return, smoothness and lengthscale. \lbps \mpc performance is reported for 1024 rollouts.  
	Uncertainty shows quartiles over dataset demonstrations or experiment seeds.
	See the discussion text for the omitted result $^\dagger$.
	}
\end{table}

\clearpage\newpage

\begin{figure}[!h]
	\centering
	\includegraphics{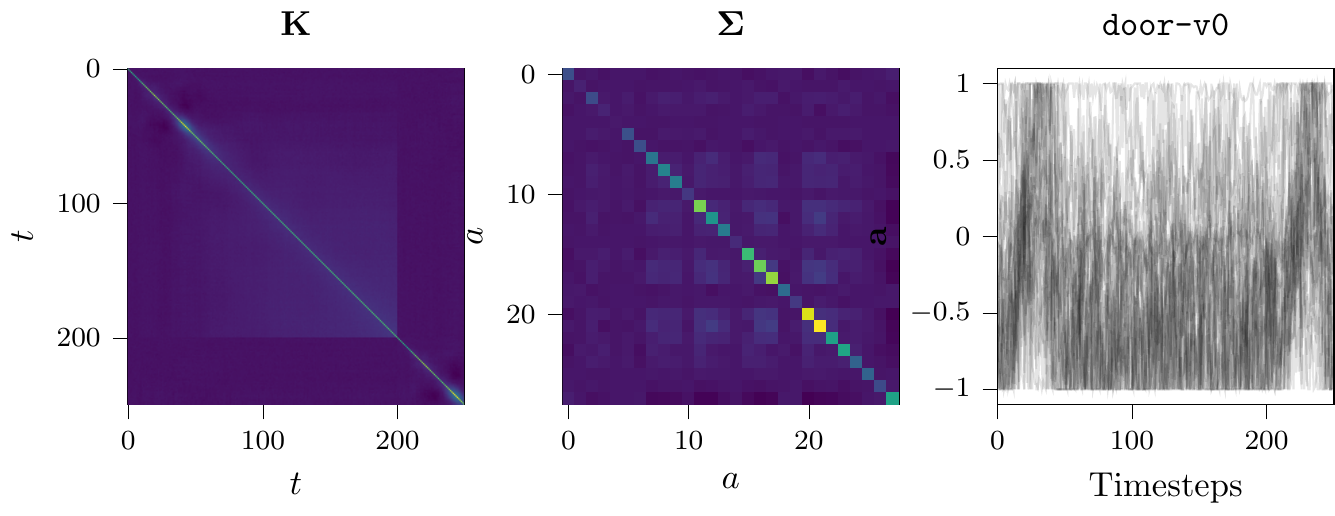}
	\caption{\texttt{door-v0} expert demonstrations, showing matrix normal covariance fit and action sequence. The \texttt{viridis} colourmap is used to express correlations.}
	\vspace{-2em}
	\label{fig:mavn_door}
\end{figure}

\begin{figure}[!h]
	\centering
	\includegraphics{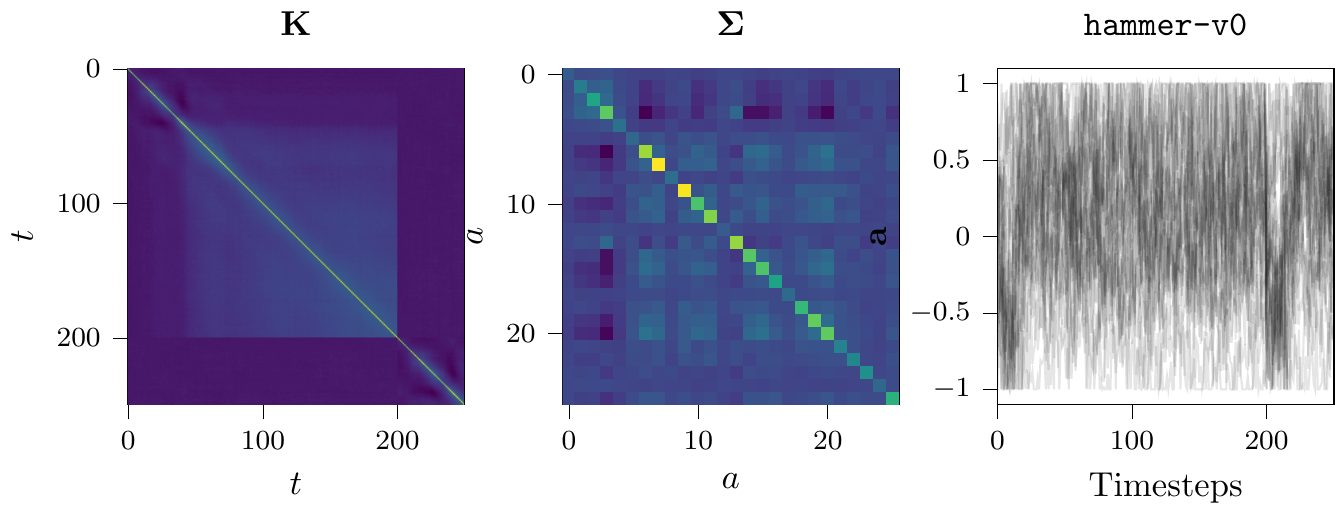}
	\caption{\texttt{hammer-v0} expert demonstrations, showing matrix normal covariance fit and action sequence. The \texttt{viridis} colourmap is used to express correlations.}
	\label{fig:mavn_hammer}
	\vspace{-2em}
\end{figure}

\begin{figure}[!h]
	\centering
	\includegraphics{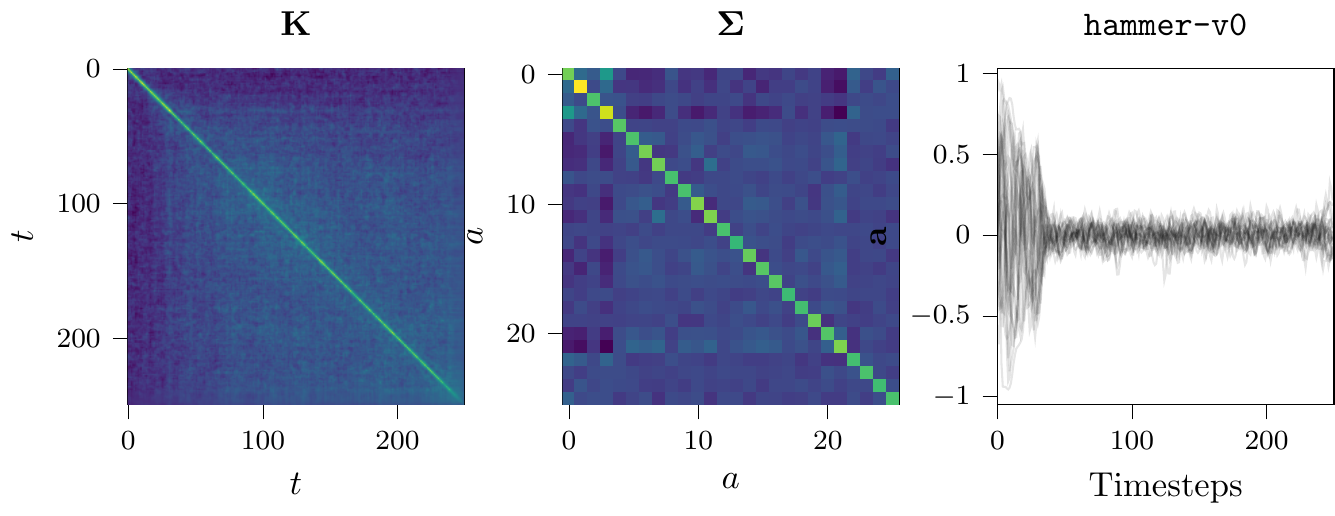}
	\caption{\texttt{hammer-v0} expert demonstrations from a \mpc solvers (\mppi, \lbps, \essps) using the \se kernel, showing matrix normal covariance fit and action sequence. The \texttt{viridis} colourmap is used to express correlations.}
	\label{fig:mavn_se_hammer}
\end{figure}

\newpage

\begin{figure}[!h]
	\centering
	\includegraphics{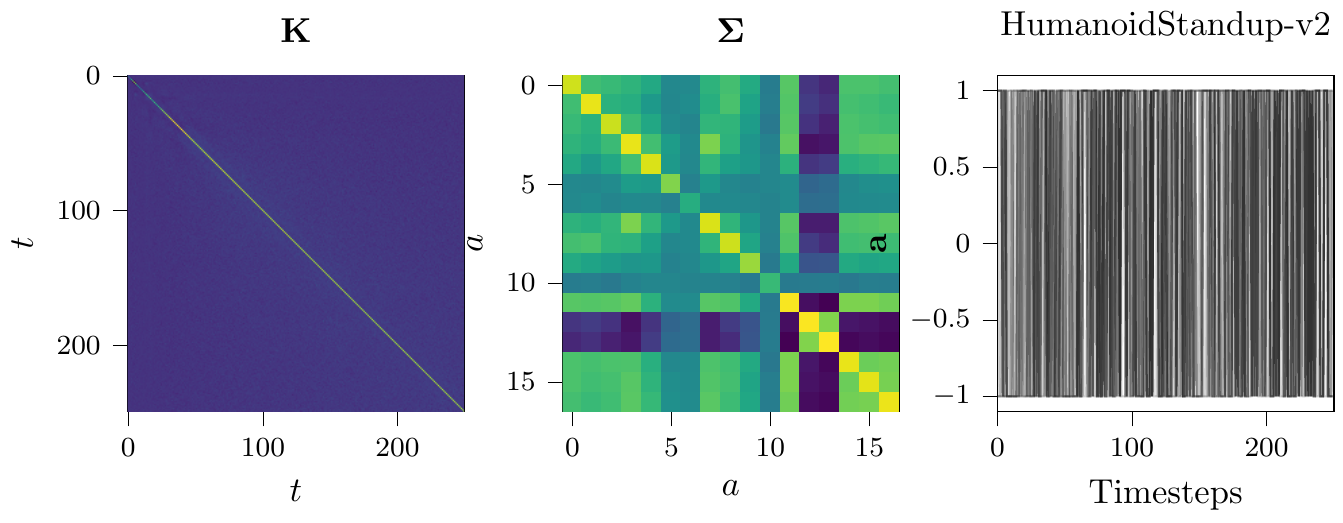}
	\caption{\texttt{HumanoidStandup-v2} expert demonstrations from a \textsc{gac} agent, showing matrix normal covariance fit and action sequence. The \texttt{viridis} colourmap is used to express correlations.}
	\label{fig:mavn_hsu}
\end{figure}

\begin{figure}[!h]
	\centering
	\includegraphics{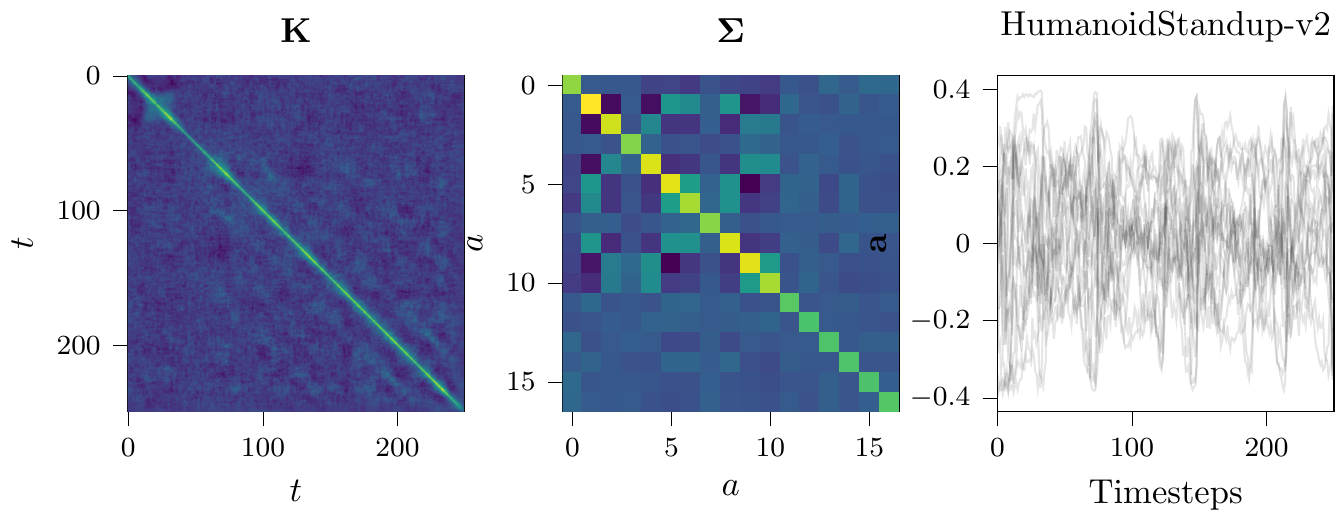}
	\caption{\texttt{HumanoidStandup-v2} expert demonstrations from a \mpc solvers (\mppi, \lbps, \essps) using the \se kernel, showing matrix normal covariance fit and action sequence. The \texttt{viridis} colourmap is used to express correlations.}
	\label{fig:mavn_se_hsu}
\end{figure}

\clearpage\newpage 

\subsubsection{Model predictive control ablation studies}
\label{sec:mpc_ablation}

The \mpc methods compared in Section \ref{sec:mpc} have subtle variations in their implementation.
\mppi has a constant covariance during optimization, while \cem \mpc resets the covariance each timestep. 
To understand the implications of these details, we provide ablations over these design decisions, in comparison to the results in Section \ref{sec:mpc}. 

\textbf{\lbps and \essps with constant covariances}

On the whole the performance drops compared to the updated covariance results (Figure \ref{fig:smooth_mpc}), which suggets that a fixed variance requires greedier optimization (i.e. \mppi) or more iterations per timestep for good performance. 

\begin{figure}[!h]
	\includegraphics{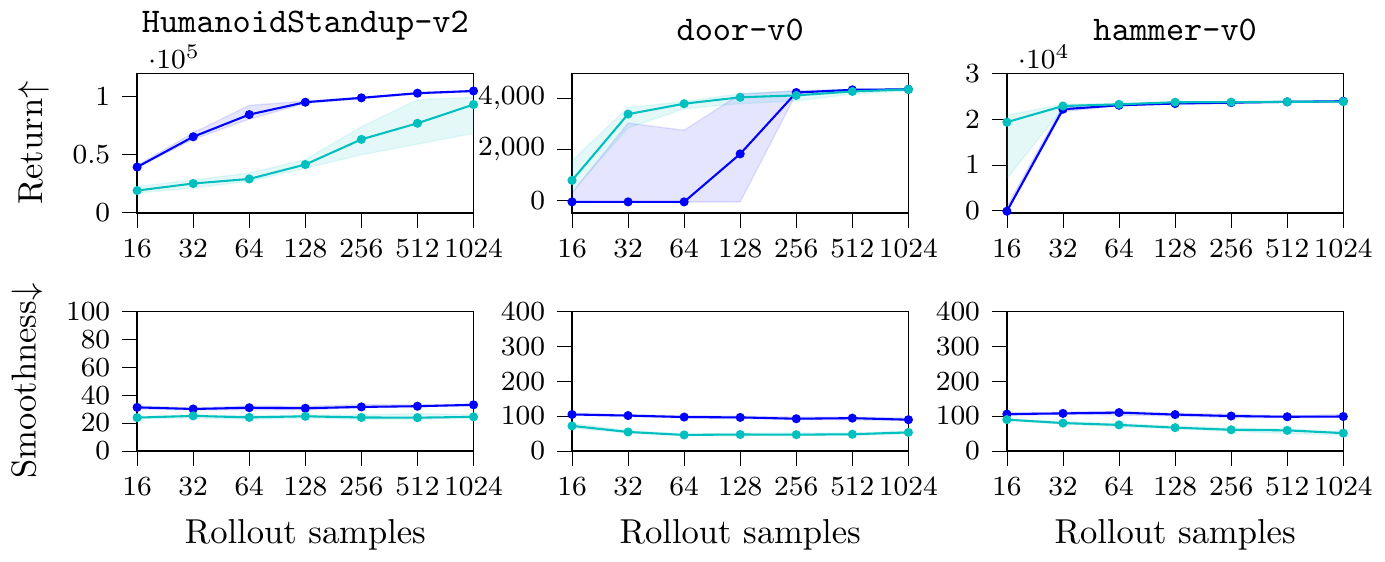}
	\caption{Monte Carlo \mpc with \lbps and \essps with fixed covariance, like \mppi, as opposed to updated each timestep. Displaying quartiles over 20 seeds.}
	\label{fig:ppi_anneal_albation}
\end{figure}

\textbf{\mppi with covariance updates}

Conversely to the result above, \mppi's greedier optimization does not work effectively with covariance updates, as entropy is quickly lost during the control. 

\begin{figure}[!h]
	\includegraphics{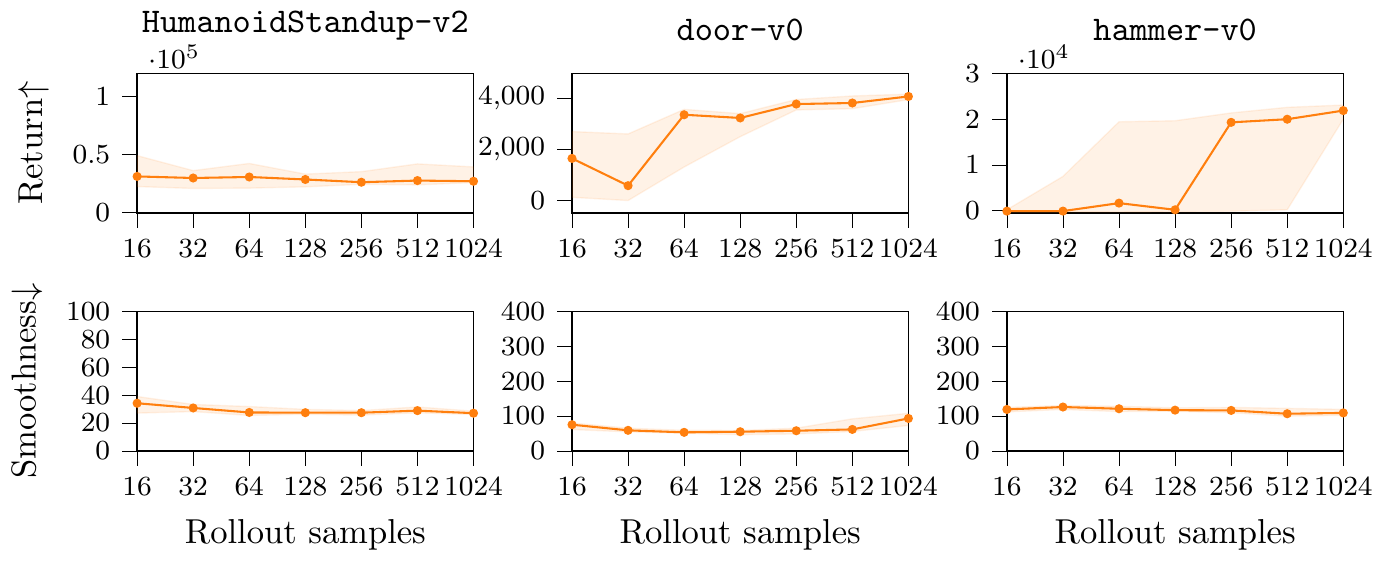}
	\caption{Monte Carlo \mpc using \mppi with covariance updates. Displaying quartiles over 20 seeds.}
	\label{fig:mppi_update}
\end{figure}

\clearpage\newpage

\textbf{\cem with a \se kernel prior}

Compared to the white noise prior in Figure \ref{fig:wn_mpc}, the \se kernel provides an improvement boost for \hsu and \hammer. 
There is a smoothness improvement, but slightly smaller than compared to Figure \ref{fig:smooth_mpc}.
Based on the results of Figure \ref{fig:ppi_anneal_albation}, we can attribute this to \cem \mpc resetting the policy covariance each timestep.  

\begin{figure}[!h]
	\includegraphics{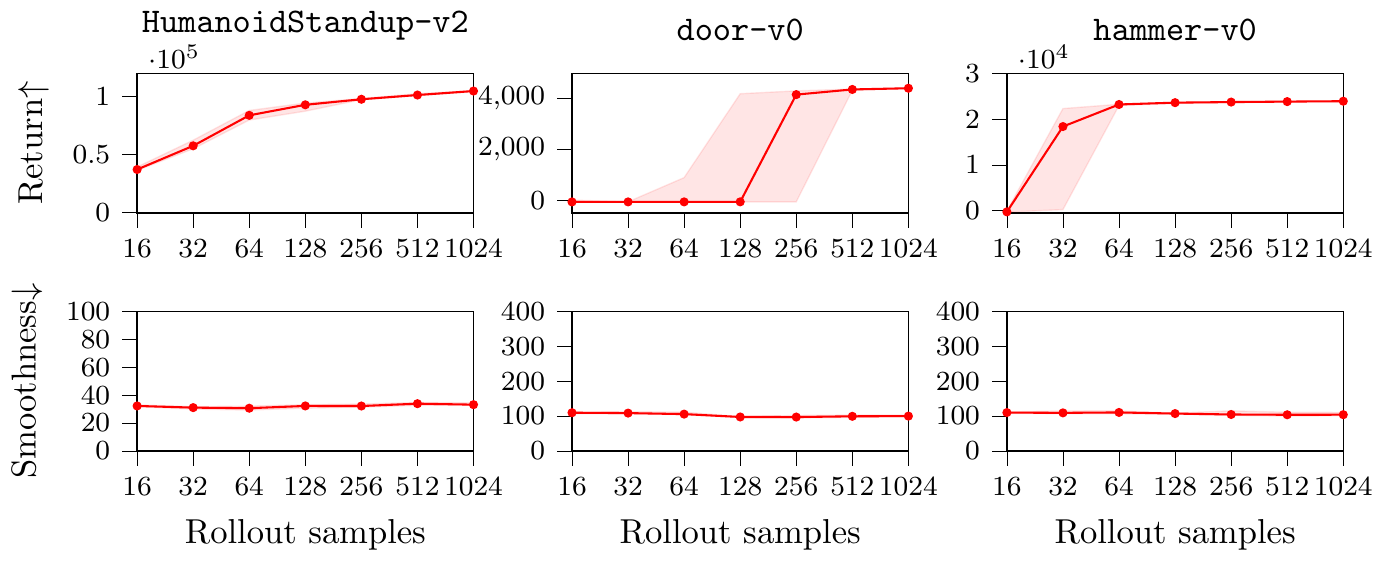}
	\caption{Monte Carlo \mpc using \cem with the \se. Displaying quartiles over 20 seeds. Results closely match \essps with the \se kernel and fixed covariance (Figure \ref{fig:ppi_anneal_albation}).}
	\label{fig:cem_se_update}
\end{figure}

\newpage

\section{Experimental Details}

\subsection{Hyperparameter Selection}
Parameters were swept with there respective solver for best average performance with 256 sample rollouts, averaged over 5 seeds.
Parameters were not swept if prior work had identified effective values.
\begin{table}[!h]
	\centering
	\caption{Hyperparameter selection for model predictive control methods}
	\label{tab:hypopt_sweep}
			\begin{tabular}{llllll}
					& $\alpha$ (\mppi) &  $\bar{\alpha}$ (\pitwo) \cite{theodorou2010generalized} & $\delta$ (\lbps) & $k$, $N^*$ (\cem, \icem, \essps) \cite{PinneriiCEM}\\
					\midrule
					Considered & [0.1, 1, 10] & - & [0.1, 0.5, 0.9] & - \\
					\hsu & 10 & 10 & 0.9 & 10 \\
					\door & 10 & 10 & 0.5 & 10 \\
					\hammer & 10 & 10 & 0.5 & 10 \\
					\bottomrule     
				\end{tabular}
\end{table}

\begin{table}[!h]
	\centering
	\caption{Hyperparameter selection for model predictive control policies}
	\label{tab:hypopt_policy_sweep}
		\begin{tabular}{lcccc}
			& $\beta$ (smooth noise) &  $\beta$ (smooth action) & $\beta$ (coloured noise) \cite{PinneriiCEM} & $l$ (\se kernel) \\
			\midrule
			Considered & [0.5, 0.7, 0.9] & [0.5, 0.7, 0.9]& - & [0.01, 0.025, 0.05, 0.1] \\
			\hsu & 0.5 & 0.7 & 2.0 & 0.05 \\
			\door & 0.9 & 0.9 & 2.5 & 0.05 \\
			\hammer  & 0.7 & 0.7 & 2.5 & 0.025\\
			\bottomrule     
		\end{tabular}
\end{table}

\subsection{Experiment Hyperparameters}
\label{sec:experiments}

\textbf{Black-box Optimization.}
We evaluated the solvers on the $20$-dimensional version of the test functions.
The initial search distribution was $\gN(\vone, 0.5\mI)$ for all methods.
32 samples were used over 20 iterations.

\textbf{Policy Search.}
The experimental setting of prior work \citep{lutter2020differentiable} was reproduced, with the same \mujoco simulation and episodic reward function.

\begin{table}[!h]
	\centering
	\caption{Hyperparameters for policy search tasks.}
	\label{tab:hypopt_ps}
		\begin{tabular}{lccccccc}
			& $T$ & $d_a$ & \texttt{n\_iter} & \texttt{n\_samples} & \rbf features & \rff order & lengthscale, $l$ \\
			\midrule
			\texttt{BallInACup} & 1000 & 4 & 40 & 128 & 20 & 10 & $\sqrt{0.03}$\\
			\bottomrule     
		\end{tabular}
\end{table}

\textbf{Model Predictive Control.}
For the \mpc tasks, the action mean and variance for each dimension $i$ was set using
\begin{align*}
\mu_i &= \frac{a_{\text{max}, i} + a_{\text{min}, i}}{2},
\quad
\sigma_i^2 = \frac{(a_{\text{max}, i} - a_{\text{min}, i})^2}{4}.
\end{align*}
This specification means that $\mu_i\pm\sigma_i$ reaches the actuator limits, ensuring coverage across the actuation range when sampling. 
$\mu_i$ defines a `mean function' applied as a bias to the policy, which had nominally zero mean.
$\sigma_i^2$ defines the diagonal of the matrix normal action covariance $\mSigma$.
To match the covariance size of the kernel, 30 features were used for the \rbf and \rff approximations, i.e. order $\nu = 15$. 

\begin{table}[!h]
	\centering
	\vspace{0em}
	\caption{Hyperparameters for model predictive control tasks.}
	\label{tab:hypopt_mpc}
		\begin{tabular}{lccccc}
			& $T$ & $H$ & $d_a$ & \texttt{n\_iters} & \texttt{n\_warmstart\_iters}\\
			\midrule
			\texttt{HumanoidStandup-v2} & 250 & 30 & 17 & 2 & 50 \\
			\texttt{door-v0} & 250 & 30 & 25 & 1 & 50 \\
			\texttt{hammer-v0} & 250 & 30 & 25& 1 & 50 \\
			\bottomrule     
		\end{tabular}
		\vspace{-2em}
\end{table}

\clearpage\newpage

\section{Characterizing the Lower-bound Objective}
\label{sec:lower_bound}
The temperature objective of \reps is the dual function, which is convex (Definition \ref{def:ereps}). 
For the lower bound introduced Section \ref{sec:ppi}, we characterize its nature to understand its suitability for optimization. 

First, we express objective in terms of functions $f$, $g$ and $h$ of $\alpha$, 
\begin{align*}
    \max_\alpha \gL(\alpha) &=
    \frac{\sum_n\exp(\alpha r_n) r_n}{\sum_n\exp(\alpha r_n)} - \lambda\sqrt{\frac{1}{\hat{N}(\alpha)}}
    ,
    \quad
    \hat{N}(\alpha) = \frac{(\sum_n \exp(\alpha r_n))^2}{\sum_n \exp(2\alpha r_n)}\\
    &=
    \frac{\sum_n\exp(\alpha r_n) r_n - \lambda\sqrt{\sum_n \exp(2\alpha r_n)}}{\sum_n\exp(\alpha r_n)}
    = \frac{f(\alpha) - \lambda g(\alpha)}{h(\alpha)}.
    \intertext{The gradient is available in closed form}
    \frac{\rd}{\rd\alpha}\gL(\alpha)
    &=
    \frac{(\sum_n w_n)(\textstyle\sum_n w_n r_n^2)
    - \lambda(\sum_n w_n)\sum_n r_n w_n^2/\sqrt{\sum_n w_n^2}
    - (\sum_n w_n r_n)^2
    + (\sum_n r_n w_n)(\lambda\sqrt{\sum_n w_n^2})}{(\sum_n w_n)^2}.
\end{align*}

This objective is not concave in $\alpha$, but is quasi-concave.
With $\alpha \in \sR_+$,
we enforce $r_n \leq 0$ and $\exp(\alpha r_n)$ is convex in $\alpha$ for $r_n\in\sR$.
We rewrite the objective using in dot product form 
\begin{align*}
    \min_\alpha \gL(\alpha) &=
    \frac{\sum_n\exp(\alpha r_n) r_n + \lambda\sqrt{\sum_n \exp(2\alpha r_n)}}{\sum_n\exp(\alpha r_n)} 
    = 
    \frac{\vw_\alpha\tran\vr - \lambda\sqrt{\vw_\alpha\tran\vw_\alpha}}{\vw_\alpha\tran\bm{1}},
\end{align*}
where
$\vw_\alpha = [\exp(\alpha r_1), \dots, \exp(\alpha r_N)]\tran$.
The term $\vw_\alpha\tran\vr$ is concave as it is a negative weighted sum of convex functions. 
The term $\vw_\alpha\tran\bm{1}$ is convex as it is a positive sum of weighted convex functions.
The remaining term $g(\alpha)$ can be shown to be convex in $\alpha$ (Lemma 3) using standard technques \citep{boyd2004convex}.
\begin{lemma}
The function $g(\alpha) = \sqrt{\sum_n \exp(2\alpha r_n)}$ is convex in $\alpha$ for $r_n\leq 0$ $\forall n$ .
\end{lemma}
\begin{proof}
    For convexity, where $\theta\in [0,1]$, $\alpha \in \sR_+$, $\beta \in \sR_+$, %
\begin{align*}
    \textstyle
    \sqrt{\sum_n \exp(2(\theta\alpha+(1-\theta)\beta) r_n)}
    &\leq
    \textstyle
    \theta\sqrt{\sum_n \exp(2\alpha r_n)}
    +
    (1-\theta)\sqrt{\sum_n \exp(2\beta r_n)}
\intertext{Starting with the right-hand term, we take interpolation term $\theta$ inside}
    \textstyle
    \theta\sqrt{\sum_n \exp(2\alpha r_n)}
    +
    (1-\theta)\sqrt{\sum_n \exp(2\beta r_n)}
    &=
    \textstyle
    \sqrt{\sum_n (\theta\exp(\alpha r_n))^2}
    +
    \sqrt{\sum_n ((1-\theta)\exp(\beta r_n))^2}
\intertext{Apply Jensen's inequality to both exponential terms, where $\exp(xy)\leq x\exp(y)$, as they are inside Euclidean norms so we can use $\sqrt{\sum_i p_i^2}\leq \sqrt{\sum_i q_i^2}$ if $p_i < q_i$,}
    \textstyle
    \sqrt{\sum_n \exp(\theta\alpha r_n)^2}
    +
    \sqrt{\sum_n \exp((1-\theta)\beta r_n)^2}
    &\leq
    \textstyle
    \sqrt{\sum_n (\theta\exp(\alpha r_n))^2}
    +
    \sqrt{\sum_n ((1-\theta)\exp(\beta r_n))^2}.
\intertext{Apply Minkowski's inequality to the LHS, where
$(\sum_i|x_i + y_i|^p)^{1/p}\leq (\sum_i |x_i|^p)^{1/p} + (\sum_i|y_i|^p)^{1/p}$}
    \textstyle
    \sqrt{\sum_n (\exp(\theta\alpha r_n)+\exp((1-\theta)\beta r_n))^2}
    &\leq
    \textstyle
    \sqrt{\sum_n \exp(\theta\alpha r_n)^2}
    +
    \sqrt{\sum_n \exp((1-\theta)\beta r_n)^2}
\intertext{Expand the terms of the LHS and remove the (non-negative) squared terms}
    \textstyle
    \sqrt{\sum_n 2\exp((\theta\alpha+(1{-}\theta)\beta) r_n)}
    &\leq
    \textstyle
    \sqrt{\sum_n \exp(2\theta\alpha r_n)+\exp(2(1{-}\theta)\beta r_n) + 2\exp((\theta\alpha+(1{-}\theta)\beta) r_n)}
\intertext{Apply Jensen's inequality again to recover the initial left hand term}
    \textstyle
    \sqrt{\sum_n \exp(2(\theta\alpha+(1-\theta)\beta) r_n)}
    &\leq
    \textstyle
    \sqrt{\sum_n 2\exp((\theta\alpha+(1-\theta)\beta) r_n)}.
\end{align*}
    
\end{proof}

With the lemma, the negative penalty term is concave as $\lambda \geq 0$.

For quasi-convexity, we require the $t$-level sets to be convex 
$\{\alpha\in\sR_+ \mid \gL(\alpha) \leq t\}, t\in\sR$.
Due to the structure of the objective, we require $f(\alpha) - \lambda g(\alpha) \leq t\,h(\alpha)$.
As $f(\alpha) \leq 0$, $-\lambda g(\alpha) \leq 0$ and $h(\alpha) \geq 0$, for $t>0$ we have the empty set, which is convex.
For $t\leq 0$, we have the sum of two concave function which are both less or equal to zero, so the set is also convex.

\section{Stochastic processes, Gaussian processes and coloured noise}
\label{sec:stoch_processes}

This section summarizes the connections between stochastic processes, smoothed noise, coloured noise and Gaussian processes to motivate the use of kernels in action priors.
For further details, we refer to Section 12.3 of S\"arkk\"a et al. \citep{sarkka_solin_2019} and Chapters 4 and Appendix B of Rasmussen et al. \citep{gpml}.  

Section \ref{sec:ppi_mpc} introduced smooth Gaussian noise processes of the form 
\begin{align*}
	n^{(n)}_t = \textstyle\sum_{i=1}^p a_i\, n^{(n)}_{t-i} + b_0\,v^{(n)},
	\quad
	v^{(n)} \sim \gN(0, 1).
\end{align*}
to sample action sequences.
This is known as a discrete-time autoregressive AR($p$) process.
The ARMA($p$, $q$) process introduces a noise history,
such that ARMA($p$, $0$) is a AR($p$) process
\begin{align*}
	n^{(n)}_t = \textstyle\sum_{i=1}^p a_i\, n^{(n)}_{t-i} + \textstyle\sum_{j=1}^q b_j\,v_{t-j}^{(n)},
	\quad
	v_t^{(n)} \sim \gN(0, 1).
\end{align*}
In continuous-time, the AR($1$) is analogous to the Ornstein–Uhlenbeck (OU) process
$$\rd n(t) = a_0\, n(t)\,\rd t + v(t)\,\rd t.$$
The OU covariance function is 
$\text{Cov}(t, t') = \sigma^2\exp(-a_0|t-t'|)$, which is also known as the exponential kernel. 
For additional smoothness we can consider higher-order derivatives, which results in the Mat\'ern family of kernels.
In stochastic differential equation form, they are written as 
\begin{align*}
	n(t) &= \mH \vf(t),
	\quad
	\rd\vf(t) = \mA\,\vf(t)\,\rd t + \mL\,v(t)\,\rd t,
\end{align*}
where $\vf$ contains $n(t)$ and its derivatives, describing the state.
The order $\nu$ of the Mat\'ern defines the size of the state space and $\mA$. 
A first-order kernel reduces to the exponential kernel.
These Mat\'ern kernels are Markovian in their state space, following a linear Gaussian dynamical system (\lgds). 
Therefore, they can be compared to the Gaussian processes used in \stomp \citep{stomp} and \gpmp \citep{Mukadam_mp}, which are also \lgds{}s but with different state space models that perform Euler integration, producing priors with Gaussian noise on the velocity, acceleration or jerk.
Extending the derivatives for the Mat\'ern kernel, as the order $\nu\rightarrow\infty$, we arrive at the squared exponential kernel ${\text{Cov}(t, t') {\,=\,} \sigma^2\exp\left(-|t-t'|^2/2l^2\right)}$. 
Comparing to the exponential kernel earlier, we observe the $a_0$ in the OU process defines the lengthscale of the covariance function.

Considering stationary covariance functions, where $\text{Cov}(t,t{+}\tau)=\text{Cov}(\tau)$, the power spectral density is defined as the Fourier transform of the covariance function 
\begin{align*}
	S(\omega) = \int \text{Cov}(\tau)\exp(i\omega\tau)\,\rd\tau.
\end{align*}
From the linear systems perspective, the parameters $a$ and $b$ of an ARMA process describe a linear filter $F(\omega)$ where
$N(\omega) = F(\omega)V(\omega)$.
In the frequency domain, such is realized as a filter
$$
F(\omega) = \frac{|B(\exp i\omega)|^2}{|A(\exp i\omega)|^2},
\quad
\text{where}
\quad
A(\omega) = \textstyle\sum_{k=1}^p a_i \exp(ki\omega).
$$
Conversely, \emph{coloured} noise with parameter $\beta$ applies a filter 
$\propto 1/\omega^{\beta}$ to Gaussian noise. 
The power spectrum of the $\nu$-order Mat\'ern kernel is $S(\omega) \propto (l^2 + \omega^2)^{\nu/2 + 1}$ and the squared exponential is 
$S(\omega) \propto \exp(-l^2\omega^2/2)$.
While there is no explicit connection between coloured noise and Gaussian processes, by reasoning about their power spectrums it can be seen that they can produce similar sampled paths.

\newpage

\end{document}